\documentclass{article}
\usepackage{graphicx} 
\usepackage[margin=1in]{geometry}
\usepackage{amsmath}
\usepackage{amsfonts}
\usepackage{amsthm}
\usepackage{amssymb}
\usepackage{mathtools}
\usepackage{xcolor}
\usepackage{algorithm2e}
\usepackage[hidelinks]{hyperref}
\usepackage{tikz}
\usepackage{pgfplots}
\usepackage{booktabs}

\newcommand{\hide}[1]{}

\newtheorem{definition}{Definition}
\newtheorem{lemma}{Lemma}
\newtheorem{theorem}{Theorem}
\newtheorem{corollary}{Corollary}

\def\R{\mathbb{R}}
\def\N{\mathbb{N}}

\def\M{\mathbf{M}}
\def\K{\mathbf{K}}
\def\TV{\mathrm{TV}}

\def\cx{x}
\def\dx{\mathbf{x}}

\def\epshold{\epsilon_{\mathrm h}}

\def\Renyi{R\'enyi}

\def\methodsvd{\texttt{GKV-ST}\xspace}
\def\methodem{\texttt{dEM}\xspace}
\def\methodemcont{\texttt{cEM}\xspace}
\def\methodkausik{\texttt{KTT}\xspace}

\def\hit{\ensuremath{\mathsf{hit}}\xspace}
\def\miss{\ensuremath{\mathsf{miss}}\xspace}

 \newcommand{\spara}[1]{\paragraph{#1}}
\newcommand{\fabian}[1]{\textcolor{red}{[Fabian: #1]}}

\def\figspace{\vspace{-4mm}}

\title{Markovletics: Methods and A Novel Application for \\Learning Continuous-Time Markov Chain Mixtures
}
\date{February 2024}
\author{Fabian Spaeh\\ fspaeh@bu.edu \\ Boston University \and Charalampos E. Tsourakakis\\ ctsourak@bu.edu \\ Boston University}
\date{February 2023}

\begin{document}

\maketitle

\begin{abstract}
Sequential data naturally arises from user engagement on digital platforms like social media, music streaming services, and web navigation, encapsulating evolving user preferences and behaviors through continuous information streams. A notable unresolved query in stochastic processes is learning mixtures of continuous-time Markov chains (CTMCs). While there is progress in learning mixtures of discrete-time Markov chains with recovery guarantees~\cite{gupta2016mixtures,spaeh2023casvd,kausik2023mdps}, the continuous scenario uncovers unique unexplored challenges. The intrigue in CTMC mixtures stems from their potential to model intricate continuous-time stochastic processes prevalent in various fields including social media, finance, and biology.

In this study, we introduce a novel framework for exploring CTMCs, emphasizing the influence of observed trails' length and mixture parameters on problem regimes, which demands specific algorithms. Through thorough experimentation, we examine the impact of discretizing continuous-time trails on the learnability of the continuous-time mixture, given that these processes are often observed via discrete, resource-demanding observations. Our comparative analysis with leading methods explores sample complexity and the trade-off between the number of trails and their lengths, offering crucial insights for method selection in different problem instances. We apply our algorithms on an extensive collection of Lastfm's user-generated trails spanning three years, demonstrating the capability of our algorithms to differentiate diverse user preferences.  We pioneer the use of CTMC mixtures on a basketball passing dataset to unveil intricate offensive tactics of NBA teams. This underscores the pragmatic utility and versatility of our proposed framework. All results presented in this study are replicable, and we  provide the implementations to facilitate reprodubility.  

\end{abstract}

\section{Introduction}
\label{sec:intro}

Continuous-Time Markov Chains (CTMCs) are a fundamental class of stochastic processes with a wide array of applications across various domains. They are particularly crucial in modeling systems where events occur continuously over time, such as in queueing theory~\cite{shortle2018fundamentals}, finance~\cite{karatzas1998methods}, understanding disease progression~\cite{liu2015efficient,beerenwinkel2004learning,prabhakaran2013hiv} and telecommunications among others~\cite{stewart1995introduction}. The inherent memoryless property of CTMCs, where the future behavior of the system is independent of the past given the present, makes them a natural choice for modeling random processes evolving over time~\cite{norris1997continuous}. In the realm of biological systems, for instance, CTMCs have been instrumental in modeling the stochastic behavior of genetic networks and the evolution of species~\cite{zhao2016bayesian}. Similarly, in finance, they are employed to model various continuous-time financial models including models for option pricing~\cite{karatzas1998methods}. CTMCs for  molecular kinetics have gained popularity in recent years, as they approximate the long-term statistical dynamics of molecules using a Markov chain over a discretized partition of the configuration space~\cite{prinz2011markov}.  It is notable that in such applications, unlike our focus where we have a set of $n$ well-defined states, elucidating the state-space description isn't straightforward.
The flexibility and analytical tractability of CTMCs, along with their capability to provide insightful analytical results, make them an indispensable item in the toolkit of researchers and practitioners dealing with stochastic systems evolving over time.  In this study, we delve into the largely untouched realm of learning mixtures of homogeneous CTMCs from trail data (refer to Section~\ref{sec:rel} for a formal definition). 

Formally, a mixture $\M$ is represented by a tuple of $L\geq 2$ Markov chains on a finite state space $[n] \coloneqq \{1, \dots, n\}$, denoted as a sequence $\M = (M^1, M^2, \dots, M^L)$.
Each chain is linked with a vector of initial probabilities, denoted as $ s^\ell \in \mathbb R^{n} $ with $ \sum_{\ell=1}^L \sum_{y=1}^n s^\ell_y = 1 $.
For discrete-time Markov chains, each chain $ \ell \in [L] $ is defined by a stochastic matrix representing transition probabilities.
The aim is to ascertain the parameters of $ \M $, encompassing the transition matrices and initial probabilities, based on observed trail data~\cite{gupta2016mixtures,spaeh2023casvd}. In the case of continuous-time Markov chains (CTMCs), each chain is now characterized by a {\it rate} matrix $K^\ell$ along the starting probabilities. The continuous-time stochastic process unfolds as follows: Initially, we find ourselves in a chain $ \ell \in [L] $ and state $ y \in [n] $ with a probability of $ s^\ell_y $. Subsequently, the transition between states is directed by the rate matrix $ K^\ell $. Specifically, in state $ y $, we select exponential-time random variables $ E_z \sim \mathrm{Exp}(K_{yz}) $ for all states $ y \not= z $. We then transition to state $ z^* = \arg\min_z E_z $ after a time duration of $ E_{z^*} $. Upon transitioning states, we reiterate the process. This mechanism generates a trail $ (x_t)_{t \ge 0} $ where $ x_t $ represents the state at time $ t $.
 
The challenge entails the recovery of a mixture of CTMCs, specified as follows:
Provided a set $X$ of 
continuous-time trails $(x_t)_{t \ge 0}$, is it possible to retrieve the rate matrices $\K = (K^1, \dots, K^L)$ and starting probabilities $(s^1, \dots, s^L)$, albeit up to a permutation?

In practical scenarios, the observation of a continuous-time process is typically approximated through discrete-time observations. A significant challenge, which inherently does not arise in the recovery of discrete-time Markov chains, entails the recovery of continuous-time chains from their discretized observations. Numerous previous works have addressed this issue, such as \cite{mcgibbon2015mle} through Maximum Likelihood Estimation (MLE) or by computing the $p$-th root of the matrix exponential~\cite{higham2009scaling}. In this study, we extend the efficient MLE approach of \cite{mcgibbon2015mle}. Intuitively, a primary difficulty lies in the fact that due to discretization, some states remain unobserved. Additional known challenges include the admissibility of the rate matrix under noise \cite{mcgibbon2015mle}. These challenges become more pronounced when recovering a mixture of CTMCs, as additional noise, for instance from incorrectly assigning a trail to the incorrect chain, comes into play. A direct formulation of the mixture problem as MLE is inefficient as solving the MLE for a mixture involves the posterior probabilities of each observed trail, leading to numerous inter-dependent terms (cf. Section~\ref{sec:proposed}). Remarkably, we demonstrate how to employ the MLE of \cite{mcgibbon2015mle} to recover the mixture while maintaining its efficiency
through
soft clusterings.

Fortunately, recent works have addressed the task of learning mixtures of discrete-time Markov chains, providing a foundation to learn the discrete-time mixture for any specified discretization interval under certain lenient conditions~\cite{spaeh2023casvd,gupta2016mixtures}. However, the hurdle of recovering mixtures of Continuous-Time Markov Chains (CTMCs) still stands: transitioning from discrete to continuous-time chains introduces additional challenges compared to the discrete-time scenario, primarily due to the variance in transition rates, as discussed in detail in Secton~\ref{subsec:discretization}.

In this work, we make the following key contributions:

\begin{enumerate}
\item We present a versatile algorithmic framework  for learning mixtures of continuous-time Markov chains through both continuous and discrete-time observations and methods to tailor it depending on the length of the trails.

\item We illustrate how the observed trails' length and the mixture parameters lead to varying problem regimes. Each regime demands distinct algorithms, for which we provide recommendations. Our recommendations rely on our theoretical results that represent new contributions and use advanced probabilistic tools, including Chernoff-like bounds for Markov chains~\cite{chung2012chernoffhoeffding}.

\item We explore the effects of discretizing continuous-time trails on the learnability of the continuous-time mixture. As continuous-time processes are usually observed through discrete, costly observations, selecting the right discretization is crucial.

\item We conduct a thorough experimental analysis, contrasting our methods with leading-edge competitors. Experimentally, we delve into the examination of sample complexity, exploring the trade-off between the number of trails and their lengths. Our discoveries offer valuable insights that can assist in selecting a suitable method for a given problem instance. We apply our algorithmic framework on user-generated trails from the Last.fm platform, demonstrating its ability to effectively capture users' listening patterns.

\item We introduce {\it  Markovletics}, a novel application of mixtures of Markov chains. We pioneer the use of CTMC mixtures on a basketball passing dataset that unveils  offensive strategies of NBA teams.  

\end{enumerate}
We defer our proofs to the appendix. 
For the sake of reproducibility, our code can be accessed publicly.\footnote{\url{https://github.com/285714/WWW2024}}





        

\begin{table}
  \centering
  \caption{Frequently used notation.}
  \medskip
  \label{tab:symbols}
  \begin{tabular}{ll}
    \toprule
    Symbol & Definition \\
    \midrule
    $n, L \in \N$ & number of states and chains resp. \\
   
    $(\cx_t)_{t \geq 0}$ & continuous-time trail with $\cx_t \in [n]$ \\
    $\tau > 0$ & discretization time parameter \\
    $\dx \in [n]^m$ & discrete-time trail of length $m$ \\
    $X$ ($X^\ell$)  & set of $r$ continuous-time trails (from chain $\ell$) \\
    $\mathbf X$ ( $\mathbf X^\ell$ ) & set of $r$ 
    discretized trails (from chain $\ell$)\\
    $c^\ell_y$ & \# of transitions
    over $y$ from trails in $\mathbf X^\ell$ \\
    $K_{\min}$ and $K_{\max}$ & min and
    max rate in the mixture $\mathbf K$ \\
    \bottomrule
  \end{tabular}
\end{table}

\subsection{Definitions and Notation}

 For  
ease of reference, we summarize the notation 
in Table~\ref{tab:symbols}. We already introduced a CTMC as a 
continuous stochastic process over
the states $[n]$ where each transition
is governed by a rate matrix
$K \in \R^{n \times n}$. For the reader's convenience, a refresher on the definition of a CTMC can be found in Appendix~\ref{sec:appendix-tau}. We can define the rate matrix
through the limit of its discretization.
Let $Y(t) \in [n]$ for $t \ge 0$ be a
random variable that holds the state
of the CTMC at time $t$.
We define
$T_{yz}(\tau) = \Pr[ Y(t + \tau) = z \mid Y(t) = y ]$
for any $t \ge 0$. In this work,
we always assume that the CTMC is
time-homogeneous, which makes $T_{yz}(\tau)$
independent of $t$, due to the memorylessness
of the process.
Note also that $T(\tau)$ is the
transition matrix of the discrete-time
Markov chain arising from observing
the CTMC at regular time intervals $\tau$.
We obtain the rate matrix as
$K = \lim_{\tau \to 0} \frac 1 \tau (T(\tau) - I_n)$
where $I_n \in \R^{n \times n}$
is the identity matrix.
By this definition, the diagonal
elements are $K_{yy} = -\sum_{z \not= y} K_{yz}$
and describe the distribution of
the time $t$ required to transition to any
state. We call the mean of $t$ the holding
time of state $y$.
Conversely, we obtain the
discretization through the
matrix exponential
$T(\tau) = e^{K \tau}$~\cite{norris1997continuous,levin2017markov,mcgibbon2015mle}.
As discussed,
a trail is a sequence
$(x_t)_{t \ge 0} = Y(t)$
for a fixed sample of the
random variable $Y$.
For a mixture $\mathbf K$,
we define as $X^\ell$
as a set of sampled trails
from $K^\ell$ and
the set of all trails
as $X = X^1 \cup \cdots \cup X^L$.

To define the distance between
two CTMCs, we look at the distribution
generating the next state. That is,
for a fixed state $y \in [n]$, we 
consider the distribution over
$(z, t)$ where $z$ is the next
state and $t$ the time of this transition.
Here, $z$ and $t$ are as defined in the
previous section through the minimum
of a set of exponential random variables.
We can evaluate the distance of
two CTMCs $K$ and $K'$ in state $y$ through
the total variation distance
(TV-distance) of the distribution
over the tuples $(z, t)$~\cite{levin2017markov}. 
The TV-distance evaluates to
\[
    \TV(K_y, K'_y)
    \coloneqq \frac 1 2 \sum_{z \not= y} \int_0^\infty | K_{yz} e^{t K_{yy}} - K'_{yz} e^{t K'_{yy}} | dt .
\]
We further define
the \emph{recovery error} between two CTMCs
as the average of the above TV distance
from all states
\[
    \textrm{recovery-error}(K, K') \coloneqq \frac 1 n \sum_x \TV(K_x, K'_x) .
\]
We then define the recovery error
between two mixtures $\mathbf K$
and $\mathbf K'$ as
the cost (wrt. the recovery-error on CTMCs) of a
minimum assignment between
the chains in the mixture
\[
    \textrm{recovery-error}(\mathbf K, \mathbf K') \coloneqq
    \frac 1 L \min_{\sigma \in S_L} \sum_{\ell = 1}^L
        \textrm{recovery-error}(K^\ell, (K')^{\sigma(\ell)})
\]
where $S_L$ is the symmetric group
of all permutations on $[L]$.

%
%
We always assume that the CTMC $K$
is irreducible, such that it has
a (unique) stationary distribution
$\pi_K$.
The mixing time $t_{\mathrm{mix}}(K)$
of a CTMC $K$ is defined as the
smallest time $t \ge 0$ such that
$\mathrm{TV}(s T(t), \pi_K) \le 1/3$.
Analogously, the mixing time $t_{\mathrm{mix}}(M)$
of an ergodic, discrete-time
Markov chain $M$ is the smallest integer $t \ge 0$
such that $\mathrm{TV}(s M^t, \pi_M) \le 1/3$.
Note that
$\lceil t_{\mathrm{mix}}(K) / \tau \rceil
= t_{\mathrm{mix}}(T(\tau))$.


\section{Related work} 
\label{sec:rel}
The exploration of Markov   chains constitutes a core subject within the realm of probability~\cite{anderson2012continuous,diaconis2009markov,levin2017markov,norris1997continuous} and computer science~\cite{jia2023online,dyer1991random,cherapanamjeri2019testing,daskalakis2018testing,fried2022identity,chan2021learning,bubley1997path,mitzenmacher2017probability}. We focus on work that lies closest to ours.  


\spara{Learning  Mixtures of Discrete Markov Chains}  The problem of learning mixtures of discrete Markov chains has been well studied in the literature.  The Expectation Maximization (EM) algorithm~\cite{dempster1977maximum,wu1983convergence} can be employed to locally optimize the likelihood of the mixture, or one could learn a mixture of Dirichlet distributions~\cite{subakan2013probabilistic} albeit without solid theoretical assurances regarding the output quality. On the other hand, moment-based techniques, leveraging tensor and matrix decompositions, can be used to provably learn (under specified conditions) a mixture of Hidden Markov Models~\cite{anandkumar2012method,anandkumar2014tensor,subakan2014spectral,subakan2013probabilistic} or Markov Chains~\cite[Section 4.3]{subakan2013probabilistic}.   
The approach introduced by Gupta et al.~\cite{gupta2016mixtures} stands out as markedly more efficient and scalable compared to the latter techniques, which are dependent on 5-trails, thereby causing their sample complexity to increase as \(n^5\) instead of \(n^3\).
Spaeh and Tsourakakis~\cite{spaeh2023casvd} delved deeper into these conditions, illustrating that they dictate constraints on the connectivity of the chains. However, they demonstrated that these constraints can be relaxed for a wider class of chains, which remain learnable despite the eased conditions.


\spara{Learning Mixtures of homogeneous CTMCs}
%
Extensive research has been conducted on learning a single continuous-time Markov chain from trails, yet transferring these techniques to handle mixtures presents a challenging and relatively uncharted domain.  
Bladt and Sørensen~\cite{bladt2005statistical} introduced an EM algorithm, along with methodologies for certain special cases, including scenarios where the rate matrix \( K \) is diagonalizable. However, this condition is often too restrictive for practical real-world applications~\cite{liu2015efficient}. Various methods have been explored and experimentally evaluated as documented by Tataru and Hobolth~\cite{tataru2011comparison}, revealing that these methods essentially compute differently weighted linear combinations of the expected values of the sufficient statistics.
On a related note, McGibbon and Pande~\cite{mcgibbon2015mle} devised an efficient Maximum Likelihood Estimation (MLE) technique to recover a single CTMC from sampled data. Two distinct variants have been proposed: one for learning a general CTMC and another for a reversible CTMC that adheres to the stated balance equations~\cite{mcgibbon2015mle,prinz2011markov}. The latter 
is tackled as a constrained optimization problem. As previously highlighted, there remains a conspicuous absence of algorithms with recovery guarantees, specifically crafted for mixtures of CTMCs that broaden the framework of Gupta et al.~\cite{gupta2016mixtures} into the continuous domain. This void is presumably a result of the inherent intricacy of the challenge, amplified by the nuanced hurdles presented by continuous-time processes compared to their discrete-time analogs. A line of inquiry that aligns closely with the current discourse is that of Continuous-Time Hidden Markov Models (CT-HMM)~\cite{calabrese2016ctmm}. Within a CT-HMM framework, both the hidden states (akin to a traditional HMM) and the transition times marking the alterations in hidden states remain unobserved. CT-HMMs manifest as a particular instance of continuous-time dynamic Bayesian networks~\cite{nodelman2012expectation}, wherein the EM algorithm~\cite{bilmes1998gentle} is employed. Luo et al.~\cite{luo2023bayesian} have advanced a Markov Chain Monte Carlo (MCMC) methodology for deducing a mixture of CT-HMMs~\cite{luo2023bayesian}. However, as elucidated by Liu et al.~\cite{liu2015efficient}, this avenue of investigation grapples with scalability constraints. In response, more scalable strategies rooted in CTMCs have been formulated by Liu et al.~\cite{liu2015efficient}.

\spara{Time parameter $\tau$}  The parameter $\tau$, also referred to as time lag or discretization parameter in other contexts, is crucial in the discretization of CTMCs. Intuitively, a ``too small" value of $\tau$ results in a scenario where no transitions are observed, while a ``too large" $\tau$ may lead to the observation of numerous transitions, many of which are not direct. In other words, by ranging $\tau$ from 0 to $+\infty$ we obtain a sequence of count matrices $C(\tau) \in \mathbb{R}^{n \times n}$  where $ c_{yz}(\tau)$ is the number of transitions from $y$ to $z$ within time $\tau$. Determining the appropriate scale for a single CTMC is a complex task for which sophisticated techniques have been devised. One common approach in molecular kinetics is the use of implied timescales, a method initially introduced by Swope, Pitera, and Suits~\cite{swope2004describing}.  Nonetheless, this method operates as a heuristic and is burdened by significant computational expenses due to the necessity of computing eigenvalues for a sequence of transition matrices $T(\tau_k=k\Delta t)$, for several integer values of the variable $k$.

\spara{Choosing the number of chains $L$}  One strategy involves utilizing model selection indices such as the Akaike Information Criterion (AIC) or the Bayesian Information Criterion (BIC)~\cite{efron1991statistical}. Another well-known strategy is the elbow method~\cite{efron1991statistical,kodinariya2013review,suhr2005principal}. A more theoretically grounded technique was introduced recently by Spaeh and Tsourakakis~\cite{spaeh2023casvd}, who leverage restrictions on the singular values of certain matrices to inform the selection of $L$. In this work, we assume $L$ is part of the input for the theory part,  but we experimentally evaluate this choice.

\section{Proposed methods} 
\label{sec:proposed}

\spara{Algorithmic framework}
In this section, we present our framework for learning 
mixtures of CTMCs, using continuous-time or discretized trails (Algorithm~\ref{alg:learning}).
Our framework comprises three stages: discretization, soft clustering, and recovery.
The advantage of this division is that each phase can be tailored independently based on the characteristics of the mixture under study and the sampling process. As a key characteristic of the latter, we use the length of the trails $m$. A comprehensive description of these three stages is provided in the following sections.  On a high level, we first discretize the continuous-time
trails by observing each trail at regular time intervals.
Note that in practice, this step may be part of the data-generating process
whenever continuous observation is impossible or too costly.
Below, we provide rules on setting the discretization parameters based on the properties of the mixture.
In the second step, we learn a \emph{soft} clustering based on the discretized trails which assigns each trail to a chain
in a probabilistic manner.
Here, we employ techniques developed for learning mixtures of discrete-time Markov chains. Finally, we use a maximum-likelihood
estimate base on the soft clustering to recover all chains.
We defer all proofs from this section to the Appendix,
where they are grouped by subsection.

\begin{algorithm}
{\bf Input:} Set of $r$ continuous-time trails $X$, discretization rate $\tau$, number of chains $L$
\\
{\bf Output:} Continuous-time mixture $\mathbf K = (K^1, \dots, K^L)$
and starting probabilities
$(s^1, \dots, s^L)$ \\
Let $\mathbf X = \{ \mathbf x = (x_{i \tau})_{0 \le i \le r} \in [n]^m : x \in X \}$ \\
Learn a soft assignment
$a \colon \mathbf X \times [L] \to [0,1]$
from $\mathbf X$ \\
\For{$\ell = 1, \dots, L$}{
    Let $s^\ell_y = \frac 1 {r}\sum_{\mathbf x \in \mathbf X : \mathbf x_0 = y}
    a(\mathbf x, \ell)$
    for all $y \in [n]$ \\
    Learn $K^\ell$ using MLE on $\mathbf X$
    with weights
    $\{ a(\mathbf x, \ell) : \mathbf x \in \mathbf X \}$ \\
}
\caption{\label{alg:learning} Framework for Learning Mixtures of CTMCs.}
\end{algorithm}



\subsection{Discretization}
\label{subsec:discretization}

A continuous-time trail $(x_t)_{t \ge 0}$ is discretized by observing it $m$  times at regular time intervals $\tau > 0$~\cite{mcgibbon2015mle}. 
That is, a sequence of states
$\mathbf x = (x_{i \tau})_{0 \le i < m}$ is obtained 
with $\dx_i = x_{i\tau } \in [n]$ for all $0 \le i < m$.
Discretization is necessitated by our methods, but it is frequently needed when the stochastic process cannot be observed continuously~\cite{mcgibbon2015mle}.  In many real-world
scenarios, observing a stochastic
process at a fixed time is costly and
thus subject to budget constraints. Given
control over the observation times,
it is thus an important task 
to optimize the discretization
parameters $\tau$ and $m$ to
enable the best possible learning
of $\mathbf K$.
To learn a mixture, two main steps are necessary: 
(1) clustering the trails and
(2) estimating the parameters of the exponential random
variables of the underlying continuous-time stochastic process. 
Concerning (1), it is crucial
for each trail to be lengthy enough
to discern the distinct model differences
among the chains within the mixture. 
We quantify this by providing
clustering guarantees for (1)
in Section~\ref{sec:clustering}.
For (2), the appropriate selection of 
$\tau$ is vital as we already explained
in Section~\ref{sec:intro} and~\ref{sec:rel}.
We now present certain criteria, independent
of the subsequent clustering method,
to choose the  discretization rate $\tau$.
Specifically,
we discuss the challenges
in choosing $\tau$ by introducing a
basic estimator for the rate matrix
$\hat{\mathbf K}$
of the underlying CTMC
under the assumption that we
have a correct clustering.
We will see how the choice
of $\tau$ and estimation quality
depends on the 
structure of the rate
matrix $\mathbf K$.
These structural challenges are
not an artifact of our estimator,
but pertain for other estimators,
as we verify experimentally
in Section~\ref{sec:exp}.

For each chain $\ell \in [L]$
and state $y \in [n]$,
we break down the estimation process of
a single row of the rate matrix
$K^\ell_y \coloneqq (K^\ell_{yz})_z$
into two phases. Initially for each state $y$ we estimate the
rate $|K^\ell_{yy}|$
(it is essential to recall that by definition in
Section~\ref{sec:intro},
the diagonal of $K$ is negative).
In order to define our estimator,
let $c^\ell_y =
  \{ \mathbf x \in \mathbf X^\ell, 0 \le i < m :
     \mathbf x_i = y \}$
be the number of times we transition through $y$
in the set $\mathbf{X}^{\ell}$ of trails from chain $\ell$.
We estimate the rate $|K^\ell_{yy}|$  
through the holding probability
$q^\ell_y \coloneqq e^{|K_{yy}| \tau}$
and thus define
the estimators
\begin{align*}
    \hat q^\ell_y &\coloneqq
    \frac 1 {c^\ell_y} |\{ \mathbf x \in \mathbf X^\ell, 0 \le i < m : \mathbf x_i = y \land \mathbf x_{i+1} = y \}|
    \quad \textrm{and} \\
    \hat K_{yy} &\coloneqq \frac 1 \tau \log(\hat q^\ell_y).
\end{align*}
With an appropriate universal choice of $\tau$ (discussed in the proof), we obtain the following estimation
guarantee:
\begin{lemma}
\label{lem:holding}
Let $0<\epshold<1$ and fix a state $y$ and chain $\ell \in [L]$.
With $c^\ell_y=\Omega\big(\epshold^{-2}\log(Ln)\big)$
transitions, our estimator $\hat q^\ell_y$ for the holding time satisfies
$ |\hat{q}_y^\ell-q_y^\ell|\le\epshold q_y^\ell $
with high probability.
\end{lemma}
Note that we consider all consecutive steps
$(i, i+1)$ as a transition,
even though there may not be a state change.
%
Second, we estimate the transition
probabilities $p^\ell_{yz}$ from $y$ to
another state $z$ through
\begin{align*}
    \hat p^\ell_{yz} &\coloneqq
    \frac{|\{ \mathbf x \in \mathbf X^\ell, 0 \le i < m :
        \mathbf x_i = y \land \mathbf x_{i+1} = z \}|}
    {|\{ \mathbf x \in \mathbf X^\ell, 0 \le i < m :
        \mathbf x_i = y \land \mathbf x_{i+1} \not= z \}|} \\
    \hat K^\ell_{yz} &\coloneqq \hat p^\ell_{yz} |\hat K^\ell_{yy}| .
\end{align*}
The quality of estimation
of the transition probabilities
is detailed in Lemma~\ref{lem:transition}
in Appendix~\ref{sec:appendix-sc}.
Our estimation critically optimizes the following
trade-off:
As we increase $\tau$,
some direct transitions become unobservable.
On the other hand, excessively reducing $\tau$
results in considerable redundancy and
challenges with numerical stability.
Informally, we desire to minimize the number skipped intermediate transitions due to the $\tau$ time resolution. We define the notion of a bad transition as follows: 

\begin{definition}[Bad transition]
\label{dfn:bad}
A pair $(x, i)$ of a continuous-time trail
$x$ and a step $0 \le i < m$ is called
a bad transition  if $x_{i \tau} = y$,
$x_{i \tau + \zeta} = z$ for
a $\zeta \in (0, 1)$,
and $x_{i \tau + \tau} = y'$,
for states $y \not= z$ and $z \not= y'$.  
\end{definition} 

Our estimators' quality deteriorates as the number
of bad transitions increases and we therefore
aim to keep the number of bad transitions small.
%
%
Clearly, the probability
to obtain a bad transition is maximized
for the state with maximum
rate $K_{\max} \coloneqq \max_{\ell, y} |K^\ell_{yy}|$,
and we can show (cf. Lemma~\ref{lem:holding})
that
\[
    \Pr[\textrm{bad transition}]
    \le \min(1, K_{\max}^2 \tau^2),
\]
which motivates
setting $\tau$ inversely
proportional to $K_{\max}$
in order to keep the fraction of
bad transitions to a small constant.
Disregarding bad transitions,
the estimation's quality as defined by Lemma~\ref{lem:holding} is primarily determined by the total duration during which we observe
the holding time without transitions. 
For one trail $x \in X^\ell$, this is as follows: 
\[
    \tau \cdot |\{ 0 \le i < m \mid x_{i \tau + \zeta} = y \textrm{ for all } \zeta \in [0, \tau] \}|
\]
We aim to maximize the
expectation of this term, over
all transitions, which by the memorylessness
of the exponential random variable $E \sim \mathrm{Exp}(|K^\ell_{yy}|)$ and the
Markov process is
\[
    \tau \sum_{x \in X^\ell, 0 \le i < m} \Pr[ x_{i \tau} = y \land E > \tau ]
    = \tau \Pr[ E > \tau ] \sum_{x \in X^\ell, 0 \le i < m} \Pr[ x_{i \tau} = y]
    = \tau \mathbb E[c^\ell_y] e^{K^\ell_{yy} \tau} .
\]
%
%
Similarly,
to prove Lemma~\ref{lem:transition},
it is essential to optimize the count of observed transitions with state changes to ensure accurate
estimation of transition probabilities.  In expectation, the number 
of transitions from state $y$ resulting in a state change is equal to 
$\mathbb E[c^\ell_y] (1 - e^{K^\ell_{yy} \tau})$ which is the lowest
for $K_{\min} \coloneqq \min_{\ell, y} |K^\ell_{yy}|$.
We detail the behavior of both quantities and the number of bad transitions
in Figure~\ref{fig:choosing-tau} (in Appendix~\ref{sec:appendix-tau}).
It is evident that we require large values of $\tau$ to efficiently capture
the exponential random variables dictating the state transitions. 
In particular,
the discrepancy in estimating holding times and transition probabilities
motivates us to  define the {\it condition number} of a continuous-time Markov chain with rate matrix $K$  as 
$  \kappa \coloneqq \frac{K_{\max}}{K_{\min}}$  where
$K_{\max} = \max_{\ell, y} |K^\ell_{yy}|$
and
$K_{\min} = \min_{\ell, y} |K^\ell_{yy}|$. To estimate the CTMC $K^\ell$  
for any $\ell \in [L]$,
and $0 < \epsilon < 1$,
we set $ \tau \coloneqq \frac{\epsilon}{100 \kappa K_{\max}} $
and obtain the following theorem: 
\begin{theorem}
    \label{thm:sc}
    Fix a chain $\ell \in [L]$ and a state $y \in [n]$. 
    If we have
    $c^\ell_y=\Omega\left(\frac{\kappa^{2}}{\epsilon^{3}}\left(n+\frac{\kappa^{2}}{\epsilon}\right)\log\left(Ln\right)\right)$
    transitions,
    we can obtain $\hat K_y^{\ell}$
    such that $\mathrm{TV}(K_y^{\ell},\hat{K}_y^{\ell})\le\epsilon$
    with high probability.
\end{theorem}
%
Consistent with our earlier discussion, 
the number of transitions needed increases with $\kappa$. 
This necessitates setting $\tau$ at a sufficiently small value to 
steer clear of bad transitions. Consequently, a larger sample set is required to effectively gauge the rates and transitions of states possessing lower rates. 
Following this, we derive the subsequent corollary, under the presumption that the number of transitions for each chain and state
are close to uniform and we are aware of the underlying chain for each trail:
\begin{corollary}
    If
    $c^\ell_x = \Omega(\frac {r m}{Ln})$,
   then we obtain an estimator
    $\hat{\mathbf K}$ of $K$ with
    $\textrm{recovery-error}(\mathbf K, \hat{\mathbf K}) \le \epsilon$
    using a total of  $r$ trails where 
    $r = \Omega\left(\frac{L n}{m} \cdot \frac{\kappa^{2}}{\epsilon^{3}}\left(n+\frac{\kappa^{2}}{\epsilon}\right)\log\left(Ln\right)\right)$
    with high probability.
\end{corollary}

In practice, we frequently do not have knowledge of the underlying chain for each transition. The subsequent section delves into strategies to address this challenge.

 \subsection{Soft Clustering}
\label{sec:clustering}
In this section, we discuss how to
assign each trail to a chain, in
a soft (i.e., probabilistic) manner.
Specifically, we aim to learn
a soft clustering
$\hat a \colon \mathbf X \times [L] \to [0,1]$
such that
$\hat a(\dx, \ell)$
is approximately proportional
to the probability of generating
$\dx$ with the $\ell$-th chain
$K^\ell$.
We denote this probability as
$\Pr[\dx \mid \mathbf K \cap \ell]$
where, in an abuse of notation,
we write $\mathbf x$ for the event
that the discretized trail
$\mathbf x$ is generated from the
mixture $\mathbf K$ and
use $\ell$ for the event that we
choose the $\ell$-th chain in $\mathbf K$.
Formally, we set
\begin{align}
\label{eq:10}
a(\mathbf x, \ell) \coloneqq
\frac{\Pr[\mathbf x \mid \mathbf K \cap \ell]}
 {\sum_{\ell'} \Pr[\mathbf x \mid \mathbf K \cap \ell']}
\end{align}
and want that
$\hat a(\mathbf x, \ell) \approx
a(\mathbf x, \sigma(\ell))$
for all $\ell \in [L]$ and $\mathbf x \in \mathbf X$
under some fixed permutation $\sigma \in S_L$.
We call such a soft clustering
(approximately) valid.
In the following Section~\ref{subsec:recovery}, we will
argue formally that a valid
soft clustering is important
for the recovery of the CTMCs.
To obtain such a valid soft
clustering, we utilize
techniques developed
for learning mixtures of
discrete-time Markov chains
and use the simple fact that
for the discretized mixture
$\mathbf T(\tau) \coloneqq (e^{K^1 \tau}, \dots, e^{K^\ell \tau})$
holds
$\Pr[\dx \mid \mathbf K \cap \ell]
= \Pr[\dx \mid \mathbf T(\tau) \cap \ell]$
which allows us to calculate
\eqref{eq:10}.

We classify problem instances
according to
properties of the mixture
and the sampling process
(i.e. the values of $r$, $m$, and $\tau$)
into different
regimes that necessitate
different approaches
to learn the soft clustering.
Naturally, for shorter
trails, we require more
difference in the transition
processes between the CTMCs of the
mixture, to be able to discern the
trails.
%
%
%
For longer trails, we can
get away with less
difference per state.
However, it is important
that the difference
in the transition process
is reflected in
the discretized mixture
$\mathbf T(\tau)$.
For instance, 
if $\tau$ is chosen close to
the mixing time in the chain,
it is only possible to
differentiate trails if
the stationary distributions
are distinct.
We introduce the
following learning
regimes categorized by different
trail lengths $m$.

\subsubsection{Short to Medium Length}
\label{sssec:short-medium-len}

For trails that are short such as those of length three or of a fixed (i.e., constant) length, or of medium length 
where $\tau m \ll t_{\mathrm{mix}}(\mathbf K)$,  encountering a state with notably distinct transition probabilities across different chains is crucial. This is vital for effectively differentiating trails from various chains. Such a state is referred to as a model difference.  As the trail length transitions from short to medium, we anticipate an improvement in the quality of the soft assignment. 

For $m=3$ (the shortest length that
allows learning a mixture \cite{gupta2016mixtures}),
we require
a model difference in every state
and non-zero starting probabilities to
observe transitions from each state.
%
In this case, we can use
singular value decomposition (SVD)
based algorithms
\cite{gupta2016mixtures, spaeh2023casvd}
to learn a
mixture
from the discretized trails $\mathbf X$,
that aims to recover
the discretized mixture
$\mathbf T(\tau)$
via an estimate $\hat{\mathbf T}(\tau)$.
From there, we
derive the soft assignment
by setting
\begin{align}
    \label{eq:11}
    \hat a(\dx, \ell) \coloneqq 
    \frac{\Pr[\dx \mid \hat{\mathbf T}(\tau) \cap \ell]}{\sum_{\ell'} \Pr[\dx \mid \hat{\mathbf T}(\tau) \cap \ell']} .
\end{align}
The following theorem establishes
the recovery guarantee, and
is based on the complex algebraic
conditions that outline the model
difference as detailed in
\cite{spaeh2023casvd}.
\begin{theorem}
    \label{thm:very-long-trails}
    If the discretized mixture
    $\mathbf T(\tau)$
    fulfills the conditions
    of Theorem~1 in
    \cite{spaeh2023casvd}, we
    can recover $\mathbf T(\tau)$
    and therefore obtain
    a valid
    $\hat a(\mathbf x, \ell)$
    from the 3-trail
    distribution
    $p_{xyz} = 
     \Pr[x_0 = x
         \land x_\tau = y
         \land x_{2 \tau} = z]$
    for all triples
    of states $x, y, z \in [n]$.
\end{theorem}
The algorithm of
\cite{spaeh2023casvd}
requires time
$O(n^5 + n^3 L^3 + L^{\mathrm{cc}})$
where $\mathrm{cc}$ is the
number of connected
components in the
mixture.
In practice,  we do not have access to the exact 3-trail distribution, but can estimate it from the transitions.
Utilizing Chernoff bounds, it can be demonstrated that  $O(n^3 \log n / \epsilon^2 )$ transitions suffice to estimate
this distribution up to $\pm \epsilon$~\cite{gupta2016mixtures}. For trails of length $m > 3$,
we use expectation maximization to learn
an estimate $\hat {\mathbf T}(\tau)$.
As in the case $m=3$, we obtain a soft
assignment from \eqref{eq:11}.
This works well in practice
especially when the number of
transitions is low, but is merely
a heuristic as convergence guarantees
of expectation for mixtures
of Markov chains are not known.

 

\subsubsection{Long Length}
\label{sssec:long-len}
If trails are sufficiently long,
we are able cluster them directly
as in
\cite{kausik2023mdps}.
Intuitively, if $\tau m \gg t_{\mathrm{mix}}(\mathbf K)$
and if the stationary distributions
are all different, we are
able to cluster the trails just by counting
the number of visits to each state.
Formally,
let $\alpha$ and $\Delta$ be such that
for all pairs of distinct chains
$K, K' \in \{K^1, \dots, K^L\}$
there exists a state $y$
such that
$\pi_K(y), \pi_{K'}(y) \ge \alpha$
and
$\| K_y - K'_y \|_2 \ge
\frac 1 \tau \Delta + 8 \tau (1 + K_{\max}^2)$.
That is, the state $y$ is visited
sufficiently often and witnesses
a model difference.
We use the algorithm of \cite{kausik2023mdps}
to obtain a clustering of the trails
which we directly use for the assignment $\hat a(\mathbf x, \ell)$.
We note that the obtained clustering
is hard, due to the long length of the trails. By Lemma~\ref{lem:kausik} in Appendix~\ref{subsec:long-trails} and \cite[Theorem 1]{kausik2023mdps}, we obtain the following result stated as a theorem: 
\begin{theorem}
    \label{thm:long-trails}
    If we have $r = \Omega(n^2 L^2 /
      \mathrm{poly}(\Delta, \alpha))$
    trails of length
    \[
        m = \Omega\left(L^{1.5} t_{\mathrm{mix}}
        \frac{\mathrm{polylog}(r)}{\mathrm{poly}(\Delta, \alpha)}
        \right)
    \]
    we 
    obtain a valid soft clustering $\hat a(\mathbf x, \ell)$
    with high probability.
\end{theorem}


Using the method
of \cite{kausik2023mdps}
requires $O(n^3 + r^2 n^2)$ time.

\subsubsection{Very Long Length}

We are able to learn the discretized chain $T(\tau)$
from only a single trail in
$\mathbf X^\ell$ for any $\ell \in [L]$
and we can obtain $\hat a(\mathbf x, \ell)$
as in \eqref{eq:11}, if the trail
is long enough.
The following theorem establishes
the length of such a trail
subject to
$t_{\mathrm{mix}}(\mathbf K)$,
the maximum mixing time in any chain,
and $\pi_{\mathrm{min}} \coloneqq
\min_{\ell, y} \pi_{K^\ell}(y)$.

%
\begin{theorem}
    \label{thm:very-long-trails}
    If $\sum_{y = 1}^n s_i^\ell = \Omega(1 / L)$ and
    we have $r = \Omega(L \log L)$ trails of length
    \[
        m = \Omega\left( \frac 1 {\pi_{\min}}
            \left( \frac n {\epsilon^2} + \frac {t_{\mathrm{mix}}(\mathbf K)}{\tau} \right)
            \log \frac n {\pi_{\mathrm{min}}}
        \right),
    \]
    then we can learn
    $\mathbf T(\tau)$ with recovery
    error at most $\epsilon$ with high probability.
\end{theorem}

\subsection{Recovery}
\label{subsec:recovery}
We now show how to recover
the individual chains in the
mixture, given the discretized
trails $\mathbf X$ and
the valid soft clustering
$\hat a(\mathbf x, \ell)$
which---depending on the chosen method
in the previous section---is
equal to or approximates $a(\mathbf x, \ell)$
up to permutation of $\ell$.
We approach this via
a Maximum Likelihood Estimation (MLE)
given $a(\mathbf x, \ell)$,
which means we
want to find
$\tilde{\mathbf K} = (\tilde K^1, \dots, \tilde K^L)$
to maximize the likelihood
of observing $\mathbf X$
under knowledge of the posteriors
$\Pr[\mathbf x \mid \mathbf K \cap \ell]$
for each $\mathbf x \in \mathbf X$.
By the law of total probability,
we can compute the probability of generating
$\mathbf X$ from $\tilde{\mathbf K}$ as
\begin{align}
    \label{eq:5}
    \Pr[\mathbf X \mid \tilde{\mathbf K}] =
    \prod_{\mathbf x \in \mathbf X} \Pr[\mathbf x \mid \tilde{\mathbf K}] =
    \prod_{\mathbf x \in \mathbf X} \sum_{\ell=1}^L
        \Pr[\ell] \cdot \Pr[\mathbf x \mid \tilde{\mathbf K} \cap \ell] 
\end{align}
To use the soft clustering and our
approximate knowledge of the posterior
$\Pr[\mathbf x \mid \mathbf K \cap \ell]$,
we try to maximize
the correlation instead of \eqref{eq:5}:
\begin{align}
    \label{eq:6}
    \prod_{\mathbf x \in \mathbf X}
    \sum_{\ell=1}^L \Pr[\ell] \cdot
        \Pr[\mathbf x \mid \tilde{\mathbf K} \cap \ell] \cdot
        \underbrace{\frac{\Pr[\mathbf x \mid \mathbf K \cap \ell]}{
        \sum_{\ell'} \Pr[\mathbf x \mid \mathbf K \cap \ell']}}_{= a(\mathbf x, \ell)}
\end{align}
The mixture $\tilde{\mathbf K}$ found by
maximizing \eqref{eq:6} serves as an approximation
to the maximizer of \eqref{eq:5}, whose
quality improves with the certainty of
the soft clustering $a(\mathbf x, \ell)$.
However, even maximizing \eqref{eq:6} is
difficult as we cannot optimize
chains individually but have to consider
their effect on the sample probability 
of each $\mathbf x$. Thus, instead of
maximizing the arithmetic mean
$\sum_{\ell=1}^L \Pr[\ell] \cdot
 \Pr[\mathbf x \mid \tilde{\mathbf K} \cap \ell]
 \cdot a(\mathbf x, \ell)$,
we consider the geometric mean\footnote{To shed light on the 
technical nuances, envision randomly picking from a mixture of two coins, each having success probabilities of \( p \) and \( q \) respectively. The resulting success probability becomes \( \frac{p+q}{2} \). Rather than maximizing the true likelihood, we maximize $p^{1/2}q^{1/2}$, which serves as a lower limit.}
\[
    \prod_{\ell=1}^L \Pr[\ell] \cdot
    \Pr[\mathbf x \mid \tilde{\mathbf K} \cap \ell]^{a(\mathbf x, \ell)}.
\]
Using this approximation, we can
rewrite \eqref{eq:6} as
\begin{align}
    \label{eq:7}
    \prod_{\mathbf x \in \mathbf X}
    \prod_{\ell=1}^L
    \Pr[\ell] \cdot
    \Pr[\mathbf x \mid \tilde{\mathbf K} \cap \ell]^{a(\mathbf x, \ell)}
    = 
    \prod_{\ell=1}^L
    \Pr[\ell]^{|\mathbf X|}
    \prod_{\mathbf x \in \mathbf X}
    \Pr[\mathbf x \mid \tilde{\mathbf K} \cap \ell]^{a(\mathbf x, \ell)}
\end{align}
In particular, we can now
optimize each chain $\ell$
individually by maximizing
the corresponding term in the RHS
of \eqref{eq:7}.
It remains to show that
\eqref{eq:7} is a good
approximation for \eqref{eq:6}.
Clearly,
when $a(\mathbf x, \ell) \to 1$
for some chain $\ell$, the
two terms also approach
equality. However, we
can even show that the terms
are close when
the entropy of $a(\dx, \ell)$
for a fixed trail $\bf x$
is high:
\begin{theorem}
    \label{thm:amgm}
    For each $\mathbf x \in \mathbf X$,
    \begin{align*}
        \prod_{\ell=1}^L
        \Pr[\mathbf x \mid \mathbf K \cap \ell]^{a (\mathbf x, \ell)}
        \le
        \sum_{\ell=1}^L
        a (\mathbf x, \ell) \cdot \Pr[\mathbf x \mid \mathbf K \cap \ell]
        \le L \cdot (\max_\ell a(\mathbf x, \ell)) \cdot
        \prod_{\ell=1}^L
        \Pr[\mathbf x \mid \mathbf K \cap \ell]^{a (\mathbf x, \ell)} .
    \end{align*}
\end{theorem}
%
%
Note that the above is tight
whenever $a(\mathbf x, \ell)$ is
uniform, over all chains $\ell \in [L]$.
This shows that our approximation
is good, for high and low entropy.
We also establish the merit of
this approximation experimentally
in Section~\ref{sec:exp}. Given the soft clustering,
we can thus use an MLE
to learn the individual chains
and their starting probabilities.
Specifically, we adapt
the iterative heuristic
introduced by \cite{mcgibbon2015mle}
to use soft assignments.
As an iterative heuristic,
the MLE of \cite{mcgibbon2015mle}
does not provide any convergence
guarantees, but performs
well in practice.
The MLE step requires
$O(n^3)$ time per iteration and
per chain
as well as scanning through
each trail to
pre-compute the transition
counts $c^\ell_y$, which requires
$O(r m)$ time.

\subsection{Customizing the Algorithmic Framework} 
\label{ssec:proposed-custom}


After presenting the three phases of our algorithmic framework for learning mixtures of CTMCs, we now describe three practical implementations that demonstrate both real-world efficiency, as discussed in Section~\ref{sec:exp}, and adherence to the previously mentioned theoretical guarantees.

\begin{itemize}
\item \methodsvd:
The  SVD-based algorithm as referenced in~\cite{gupta2016mixtures,spaeh2023casvd} is employed to learn mixtures of discrete-time Markov chains, leading to the soft clustering detailed in Section~\ref{sssec:short-medium-len}.
Given that SVD-based techniques are tailored for discrete-time
chains of length 3, we subdivide each discretized trail
into segments of this length prior to the clustering phase. 

\item \methodem: In lieu of the SVD-based algorithm, we use
expectation maximization to learn a mixture
of discrete-time Markov chains.

\item \methodkausik: We exclusively employ spectral clustering solely for the assignment step, as outlined in \cite{kausik2023mdps}, with the underlying algorithm being credited to Vempala and Wang in their work~\cite{vempala2004spectral}.  This method is elaborated on  in Section~\ref{sssec:long-len}. It is noteworthy that this method ensures hard clustering, as per its algorithmic design. 
\end{itemize}


\section{Experimental Evaluation} 
\label{sec:exp}


In our experimental analysis, we aim
to answer the following key questions
with
experiments on synthetic data:

\begin{enumerate}
\item What are the practical boundaries of the problem regimes, and how do different soft clusterings impact the performance of the algorithm?  
We investigate this in Figure~\ref{fig:error-trail-length} by varying
the trail length.
         
\item How accurate is the soft clustering and how much
of the recovery error is attributed to error in the clustering?
We examine this by monitoring the clustering error over varying trail length
$m$ in Figure~\ref{fig:error-isolate}.
In Figure~\ref{fig:error_n_observations}, we show the recovery error while maintaining a constant number of transitions
(i.e.,  $r \cdot m$ is constant); if we treat the number of transitions as a constant, the recovery error only depends on the error in the soft clustering.
        
\item  How much error is attributed to the  recovery? 
Figure~\ref{fig:error-tau} shows the recovery error across various values of~$\tau$.
\end{enumerate}

We also apply our algorithms on two real-world scenarios, user trails on Last.fm and an NBA passing data set (public yet proprietary dataset obtained from Second Spectrum player tracking).  Our results suggest that no single method is universally superior; however, \methodem consistently performs well across various trail lengths, while \methodkausik excels with extended trail lengths at a higher computational cost.


\subsection{Experimental Setup}


\spara{Synthetic Data}
We construct an underlying mixture of $L$
CTMCs as follows:
for every chain $\ell \in [L]$,
we randomly select a uniform rate  matrix $K^\ell$ from the set of all rate matrices
$K \in \R^{n \times n}$ with
entries $K_{yz} \in [0, 1]$ for all distinct states $y,z$.
Additionally, we randomly determine 
the starting
probabilities $s^\ell \in \R^n$, drawing 
uniformly from the set of starting probabilities
that sum up to $1$ over $\ell \in [L]$ and $y \in [n]$.
%
We sample $r$ continuous-time trails from the mixture according to the stochastic process described in Section~\ref{sec:intro}.
By monitoring the CTMC at consistent time intervals of $\tau$, we obtain discretized
trails $\mathbf x = (x_{i \tau})_{0 \le i < m}$.
In this process, only the first $m$ observations are retained.

\spara{Last.fm dataset}
We obtain user trails from the
Lastfm-1k dataset which can be accessed via the provided link~\cite{zenodo}.
This dataset captures users' music listening history over three years, detailing each track played with associated 
user information,
song title, and   timestamp. 
For our study, we interpret a continuous sequence of songs listened to by a user as a single trail, provided that there are no interruptions exceeding 15 minutes. To streamline our data, we limit the states to the 10 most frequently listened songs in the dataset and focus on users boasting the highest trail count.
This results in a total of 2763 continuous-time trails with an average of 9 minutes
of listening history.
We then convert the trails using a time frame of $\tau=10$ seconds.
Following this, the converted discrete trails are segmented into smaller trails, each with a length of 10. 
We obtain 13\,615 discrete-time trails.

\spara{NBA Dataset}
The NBA dataset from Second Spectrum  archives every pass executed during NBA basketball games for the 2022 and 2023 seasons.   Each documented pass is linked with a specific offensive opportunity and is marked with the time it was made, as well as 
the passer and the receiver. In this context, an opportunity refers to a continuous duration when a team possesses the ball. This record also mentions the team on the offense and the points they score during the possession.   It is important to note that in NBA rules (rule no 7), a team's opportunity to score is constrained to 24 seconds due to the shot clock regulation. This comprehensive dataset 
consists of 
1\,433\,788 passes made within 535\,351  opportunities  spanning 2\,460 games.

To dissect the data for each team during the two seasons, we employ the following approach. We designate a state for each of the top 12 players who have the ball for the longest durations.  In addition, we introduce
two special distinct states \hit
and \miss. These states signify whether an opportunity culminated in the offensive team scoring or failing to score, respectively.  
The continuous-time trail begins at the state that corresponds to the player who receives the opening pass.
The subsequent state is the receiving
player of the next pass, and
the transition time mirrors the time lapse between the two passes. 
When an opportunity wraps up, the trail concludes (i.e., is absorbed in terms of Markov chains) in the $\hit$ or $\miss$ state, contingent on the scoring outcome.    
Our analysis only includes opportunities that span beyond 5 seconds and involve a minimum of 3 passes. To ensure balance, we exclude any surplus of $\hit$ and $\miss$ opportunities.
We use trails whose total time duration is between 10 and 20 seconds. 
This leaves us with a total of 3850 trails per team on average.

\spara{Algorithms}
%
We use the three algorithms elucidated in   Section~\ref{ssec:proposed-custom}:
\methodsvd, \methodkausik,
and \methodem.
For \methodem, we limit the discrete-time expectation maximization 
algorithm to a maximum  of 100 iterations, typically ensuring adequate convergence. 
For synthetic data, we discretize with
$\tau=0.1$ unless otherwise specified.

In the second phase, when the objective is to recover a CTMC for a cluster of trails, we invariably opt for an approximation approach to the
maximum likelihood estimator
as proposed by \cite{mcgibbon2015mle} for a single chain.  Given our CTMC mixture scenario's choice of soft cluster assignments, we have modified the method to cater for a weighted set of trails, as discussed in Section~\ref{sec:proposed}.
For comparative analysis, we use  continuous-time expectation
maximization \methodemcont that is given access to  an initial span of $\tau \cdot m$  time for each continuous-time trail.
We also attempted to train  a mixture
of CTMCs by employing the Python library \texttt{HMMs}
on continuous-time hidden Markov models. Nevertheless, even with modest examples (such as when $L = 2, n = 5$), these attempts failed to reach a coherent mixture. Therefore, we omit these results from our presentation.

\spara{Evaluation Metrics} We evaluate the quality of the mixtures we have acquired using the recovery error, detailed in Section~\ref{sec:intro}. To gauge the effectiveness of the soft clustering, we present the \emph{clustering error} as:

\[
    \textrm{clustering-error}(a, a_{\mathrm{gt}})
    \coloneqq
    \frac 1 { 2 |\mathbf X|}
    \min_{\sigma \in S_L}
    \sum_{\ell = 1}^L
    \sum_{\mathbf x \in \mathbf X}
    |a(\mathbf x, \ell) -
     a_{\mathrm{gt}}(\mathbf x, \sigma(\ell))|
\]
In this context, $a$ is the soft clustering derived from our algorithms, while $a_{\mathrm{gt}}$ signifies the ground truth. Specifically, $a_{\mathrm{gt}}$ ($\mathbf x, \ell) = 1$
if $\mathbf x \in \mathbf X^\ell$ and $0$, otherwise.

%

\spara{Machine specs}
We developed our software using Python~3 and executed it on a system powered by a 2.9 GHz Intel Xeon Gold 6226R processor, equipped with 384GB RAM.

\subsection{Synthetic Experiments}

To highlight the differences in our algorithms, we study situations with different trail lengths and numbers of transitions.
In all experiments, we run each algorithm 5 times and report mean and standard deviation.

\begin{figure}
    \centering
    \includegraphics[width=0.9\linewidth]{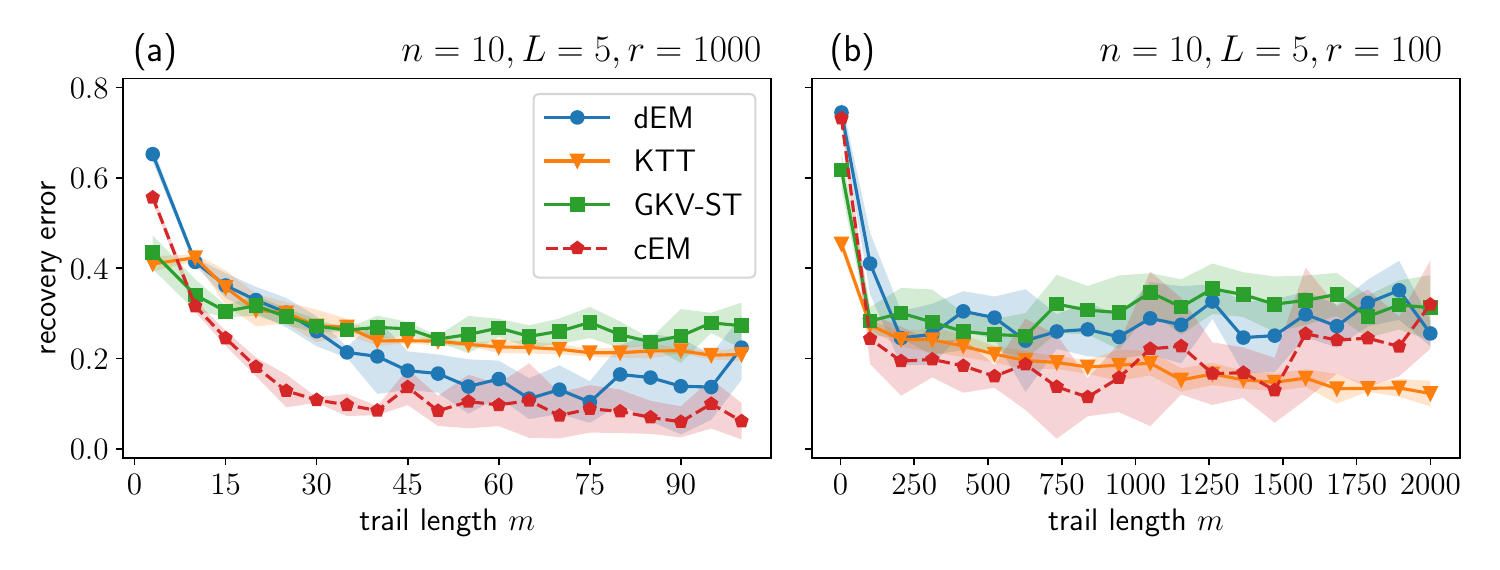}
    \figspace
    \caption{
    Recovery error across different trail lengths: The plot illustrates two distinct scenarios: (a) A large number of transitions with shorter trails, and (b) a small number of transitions with long trails.}
    \label{fig:error-trail-length}

    \bigskip

    \centering
    \includegraphics[width=0.9\linewidth]{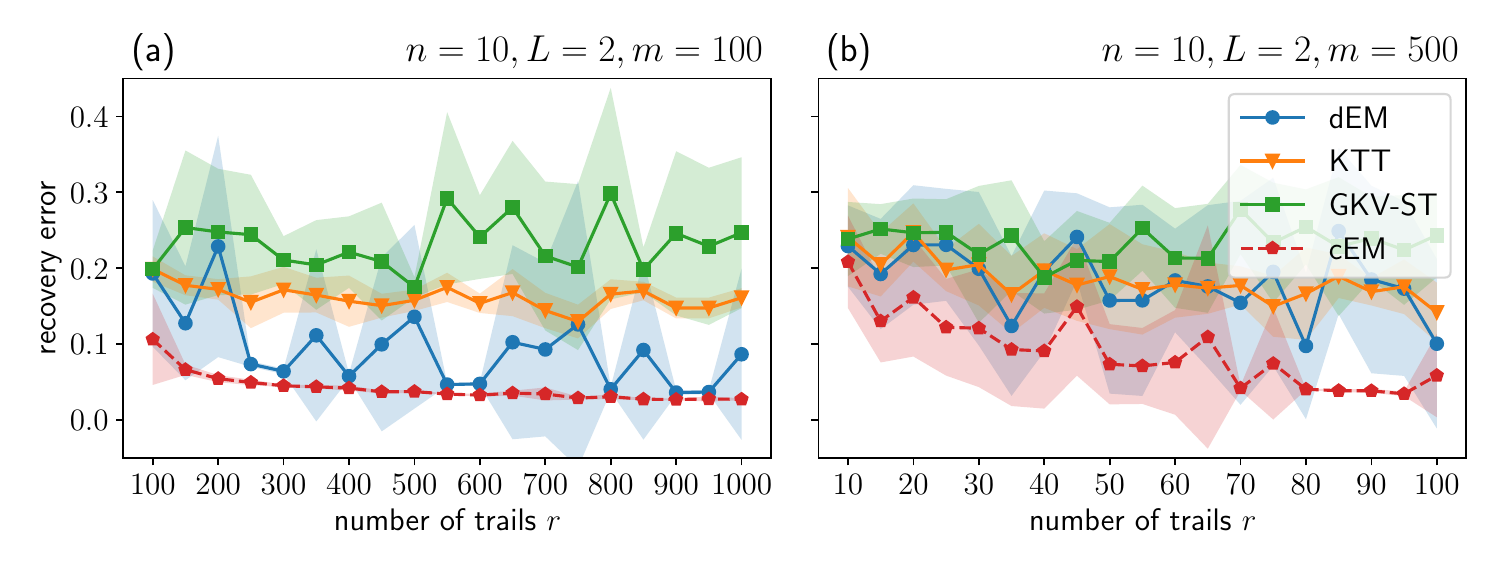}
    \figspace
    \caption{Recovery error under a
    varying number
    of trails. We consider trails
    of medium length in (a) and trails of long length in (b).}
    \label{fig:error-n_samples}

\end{figure}
\spara{Varying Trail Length and Number of Transitions}
Figure~\ref{fig:error-trail-length} presents the recovery error for our three proposed algorithms in scenarios with (a) abundant medium-length trails  and (b) limited extended-length trails. In scenario (a), \methodem mirrors the performance of \methodemcont, achieving minimal recovery error given adequate samples. Conversely, scenario (b) highlights \methodkausik's superior performance to \methodem with extended trails, attributed to its enhanced clustering capabilities. However, as trails lengthen, calculating expectations becomes less stable. Notably, \methodsvd is restricted to trails of length 3, yet excels when supplied with numerous transitions. 
This trend is further validated in Figures~\ref{fig:error-n_samples} and \ref{fig:error_n_observations} by maintaining a consistent total number of observations at (a) 2500 and (b) 25000  and contrasting medium and long trail scenarios respectively.


\begin{figure}
    \centering
\end{figure}

\spara{Scalability}
Figure~\ref{fig:scalability}~(a) shows the
scalability of our algorithms and the
\methodemcont baseline. We observe that
\methodkausik scales poorly in $n$.
On the other hand, \methodsvd and \methodem are
much faster in practice (even though the former
scales with $n^5$) and clearly outperform
\methodemcont.
Figure~\ref{fig:scalability} also shows
the running times of
\methodem, \methodsvd,
\methodkausik, and \methodemcont.
We vary
the number of
chains $L$ in (b) and the 
trail length $m$ in (c). \methodkausik scales worse than the rest of the methods. The fastest is based on the GKV-ST methodology for discrete chains~\cite{gupta2016mixtures,spaeh2023casvd} with the caveat that the theoretical conditions of the recovery theorems may not always apply, as we observed in some of our experiments.

\spara{Clustering Error}  We generate a random mixture comprised of $L=2$
chains and $n=10$ states with $\tau=0.1$. In Figure~\ref{fig:error-isolate}, we plot the clustering error as the trail length varies. We observe that 
\methodkausik is able to  enhance the clustering quality as the length of the trails increases. However, \methodem's clustering  does not benefit from the use of  trails longer than 200. We also observe that the variance of \methodem is larger than \methodkausik. 
Our methods
show differing behavior across problem regimes
depending on $m$, $r$,
but also in terms of scalability.
The methods selection is therefore problem specific.

\begin{figure}
    
    \centering
    \includegraphics[width=1.0\linewidth]{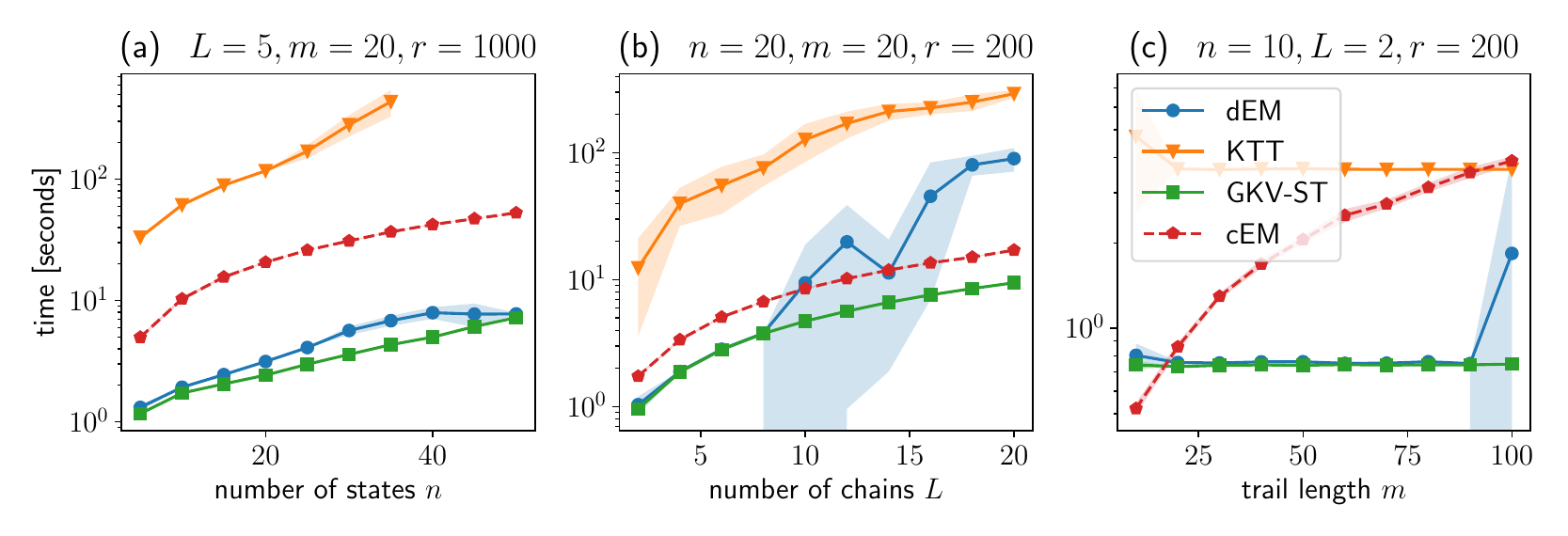}
    \figspace
    \vspace{-1em}
    \caption{Running times for varying number of states $n$ (a), chains $L$ (b), and varying trail length (c).}
    \label{fig:scalability}

    \bigskip
    
    \centering
    \includegraphics[width=0.5\linewidth]{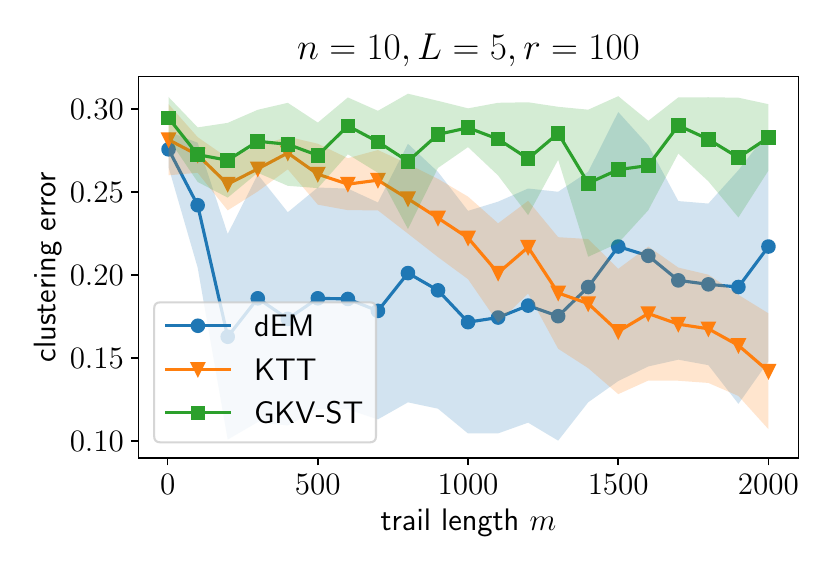}
    \figspace
    \caption{The clustering error for varying trail length $m$.}
    \label{fig:error-isolate}
\end{figure}

\spara{Varying $\tau$}
Figure~\ref{fig:error-tau} shows the effect of using different  $\tau$ values 
on the recovery error.  Note that \methodem,  \methodkausik, and \methodsvd utilize $\tau$ as a discretization parameter. In contrast, \methodemcont operates on trails generated by the CTMC, continually observed for a duration of $\tau \cdot m$.
We can see that, as the observation duration increases, the performance of \methodemcont improves with larger value of $\tau$ but also has the most variance compared to the other methods. For \methodem and \methodkausik, there exists an optimal value for $\tau$. We observe from Figures~\ref{fig:error-tau}(a) and (b) that this $\tau$ value varies as we vary the trail length $m$. This optimal value strikes a balance between observing each transition sufficiently often to ensure effective clustering and keeping $\tau$ sufficiently small to achieve optimal recovery, represented  by the mapping  
$e^{K \tau} \mapsto K$.

\spara{Ranging the trail  length $m$} 
 Figure~\ref{fig:error_n_observations} examines the behaviors of \methodem, \methodkausik, and \methodsvd, maintaining a consistent total number of observations $m \cdot r$ of (a) 2500 and (b) 25000. For longer trails, \methodkausik surpasses even the \methodemcont baseline. However, (a) also underscores \methodem's reliance on adequately longer trails, a criterion not necessary for its continuous counterpart.

\begin{figure}
    \centering
    \includegraphics[width=0.9\linewidth]{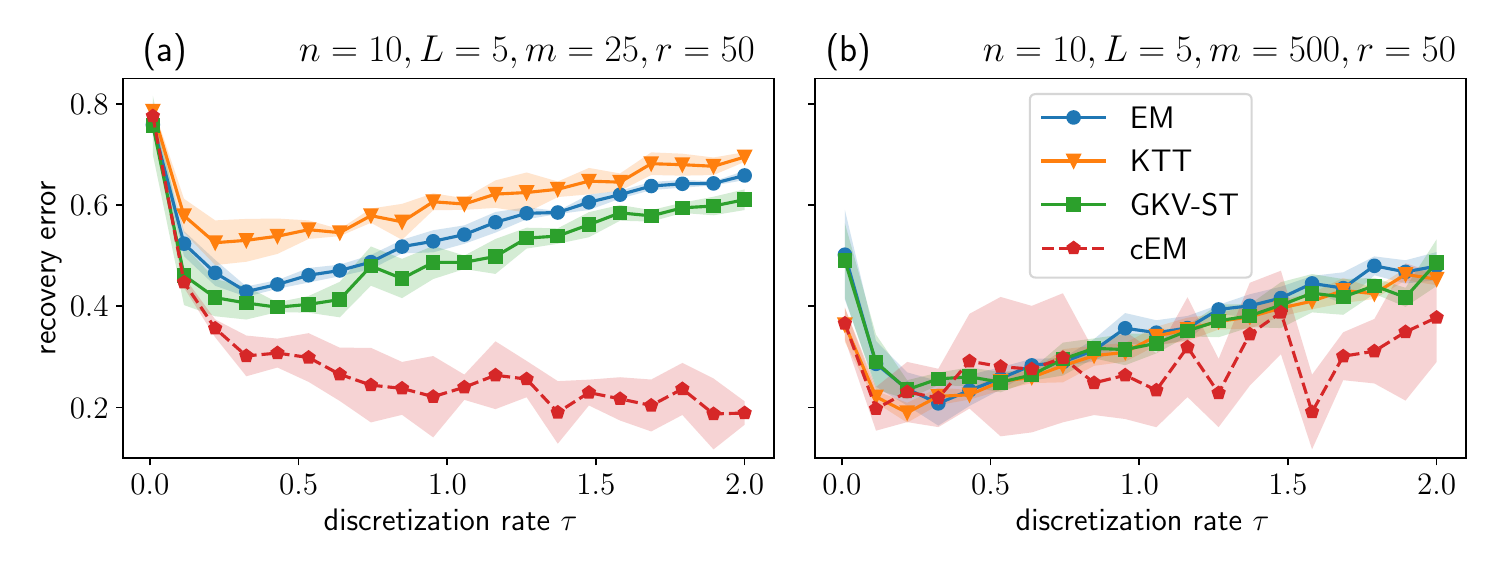}
    \figspace
    \caption{Recovery error   for different discretization rates $\tau$: (a) 20 samples with 25-length trails and (b) 100 samples with 200-length trails.}
    \label{fig:error-tau}

    \bigskip
    
    \centering
    \includegraphics[width=0.9\linewidth]{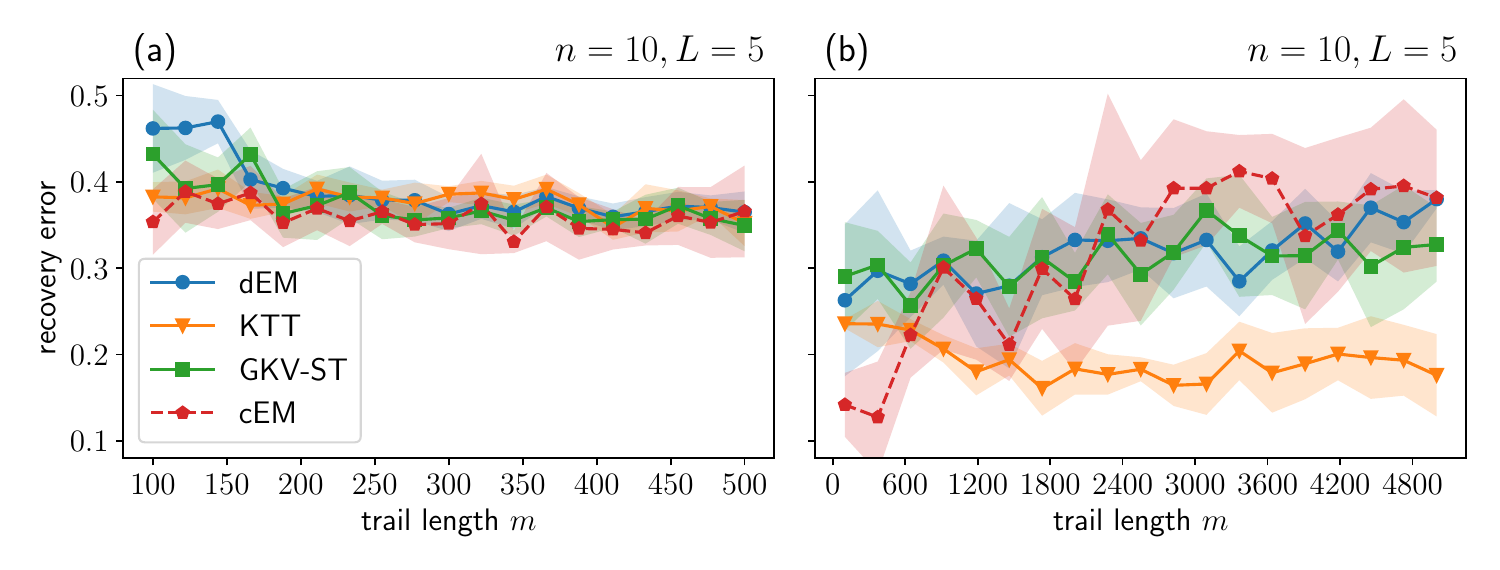}
    \figspace
    \caption{Recovery error based on transition counts: The plots present the error as a function of the total transitions, derived from multiplying the number of samples with the trail length. Two specific cases are highlighted: (a) $r \cdot m = 2\,500$ total transitions, and (b) $r \cdot m = 25\,000$ total transitions. The  figure titles contain the remaining parameters, $\tau$ is set to 0.1.}
    \label{fig:error_n_observations}
\end{figure}

\spara{Ranging the number of trails $r$} Figures~\ref{fig:error-n_samples}(a) and (b) show the recovery error for all methods as the number of samples increases from 100 to 1000 with a step of 50 and from 10 to 100 with a step of 5 respectively. The values for $n,L,m$ are 10, 2, and 500, respectively. We observe that \methodemcont has the lowest recovery error in all cases. The performance of the other three methods alternates with \methodem ranking as the second best. In the regime where we see few trails, the variance of all methods increases. 

\begin{figure}[t]
    
    \centering
     \includegraphics[width=0.9\linewidth]{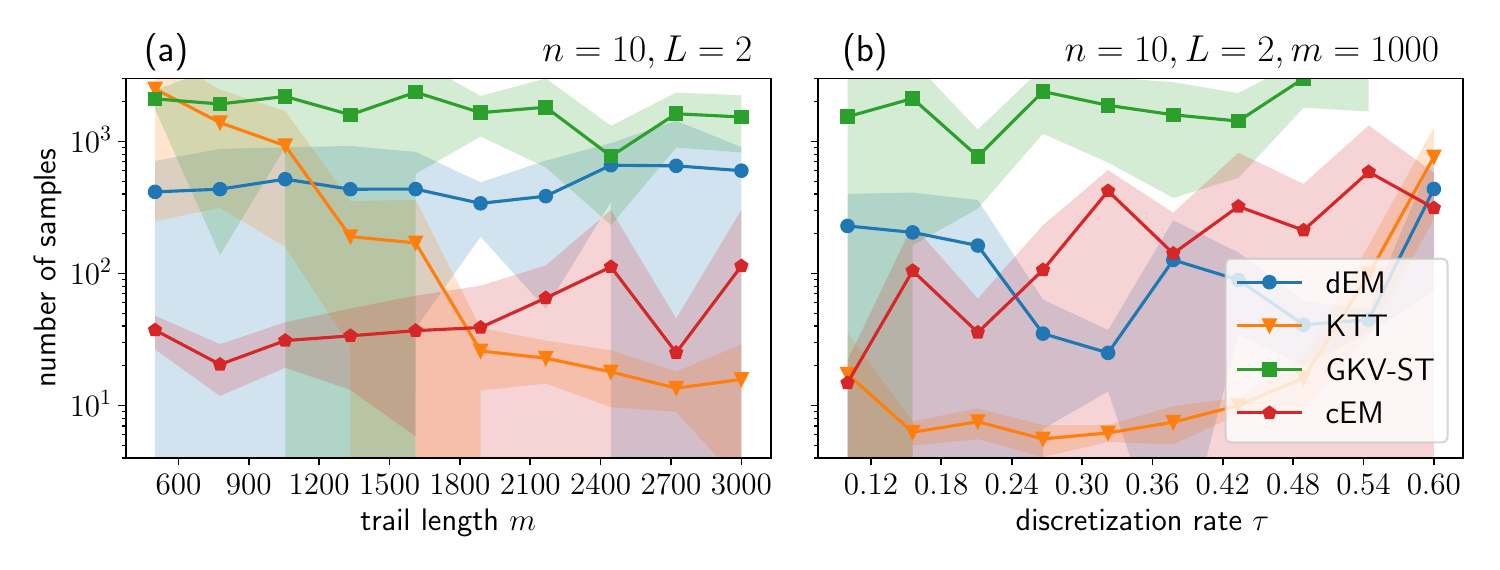}
    \figspace
    \caption{Sample complexity for a
    varying number of samples (a) and
    a varying discretization rate
    $\tau$ (b).}
    \label{fig:sample-complexity}

    \bigskip

    \centering
    \includegraphics[width=0.9\linewidth]{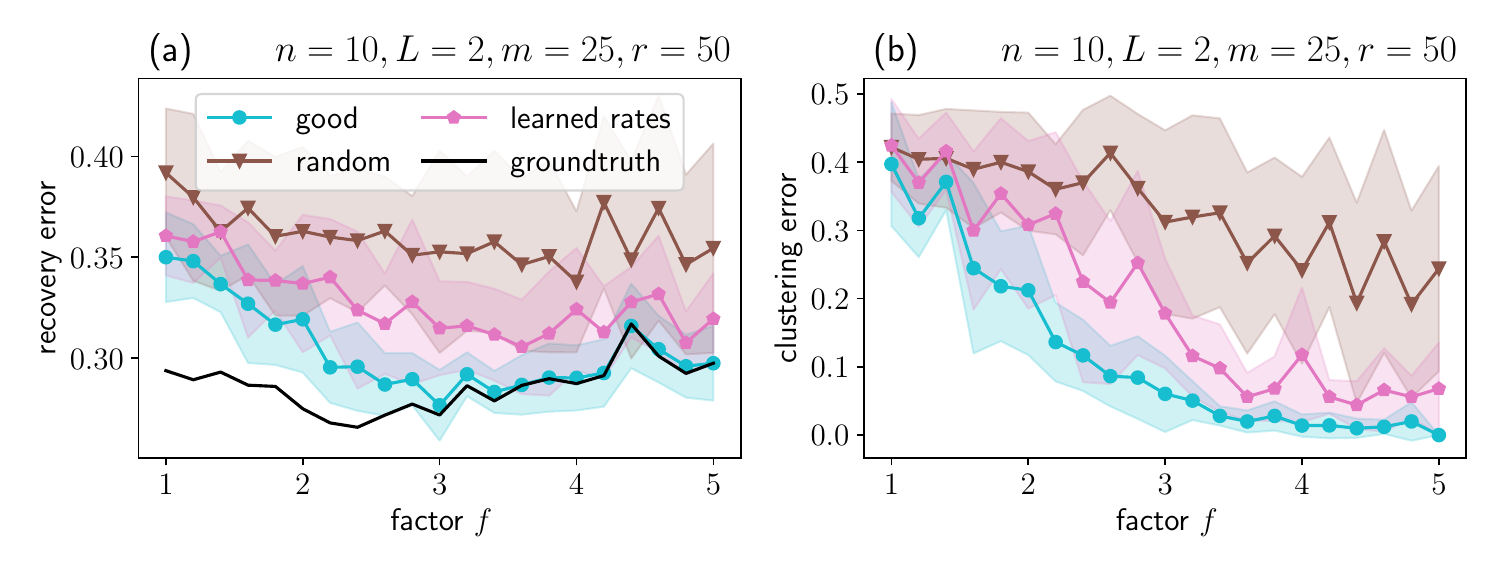}
    \figspace
    \caption{Mixture of $L=2$ chains with $K^1 = f \cdot K^2$. We show the effect of different initializations, as explained in the text.}
    \label{fig:proportional-rates}
    
\end{figure}

\spara{Sample Complexity}
For \methodem
and \methodkausik,
we study the empirical
sample complexity, i.e.
the number of samples
required to obtain
a recovery error
below a certain
threshold.
For our results,
we use a threshold of $0.1$.
Figure~\ref{fig:sample-complexity}
shows the sample
complexity for varying
trail length (a) and
varying discretization
rates $\tau$ (b).
We can clearly observe
that for increasing
trail length, \methodkausik
performs better while
\methodem performs
worse. This behavior is attributed to the local optimization nature of the EM algorithm. 
Furthermore, (b) shows
that \methodkausik still
achieves low recovery error
if trails are long enough,
even for low $\tau$, compared
to \methodem.



\spara{Proportional Rates}
We consider the difficult case when
the rate matrices $K^1$
and $K^2$ of a mixture
are proportional, i.e. there exists
a factor $f > 0$ such that
$K^1 = f \cdot K^2$ for the same graph topologies.
This is a difficult case as the discretization step may conflate the two chains into the same discrete chain.
In this hard case, we found that \methodem performed best. 
We thus use \methodem with several initializations.
For clarity, let us denote with
$\mathcal K_{[0,f]}$ the
uniform distribution over rate matrices
$K$ with $K_{yz} \in [0,f]$ for states $y \not= z$.
First, we initialize with a random mixture
sampled from $(\mathcal K_{[0, 1]}, \mathcal K_{[0,f]})$.
We call this initialization \textsf{good}.
Second, we try to learn the holding times first
and initialize \methodem with random rate matrices that have
the learned holding times (\textsf{learned}).
Third, we sample both initial rate matrices from $K_{[0,1]}$
(\textsf{random}).
We observe that the recovery and clustering error is
almost as good as when using the groundtruth clustering,
after using the good initialization or learning the
holding times.

\subsection{Real-world Experiments} 

\spara{Last.fm} 
We select $k$ users from the {\it last.fm} dataset who have generated the most trails, where $k\in \{2,5,8,\ldots, 20\}$.  We apply \methodem, \methodkausik and \methodemcont by setting the chain count $L$ equal to $k$.
We were not able to apply \methodsvd
as the conditions of \cite{spaeh2023casvd}
are violated.
This configuration inherently sets up a classification task, which is to classify the trails based on the originating user. Figure~\ref{fig:last-fm}(a) plots the average classification error for both the train (depicted by dashed lines) and test dataset calculated across five 80\%-20\% train-test splits of the entire dataset.  Using the hard clustering of \methodkausik we obtain the best possible classification error. We also observe that this can be attributed to the good performance of the clustering step.
Figure~\ref{fig:last-fm}(b) shows the median entropy of the assignment of each trail to the chains. We observe that \methodkausik even on the test data performs a low-entropy assignment, suggesting a near  hard clustering.
Interestingly, \methodem produces an assignment that tracks the entropy of the groundtruth. 

\begin{figure}[t]
    
    \centering
    \includegraphics[width=0.9\linewidth]{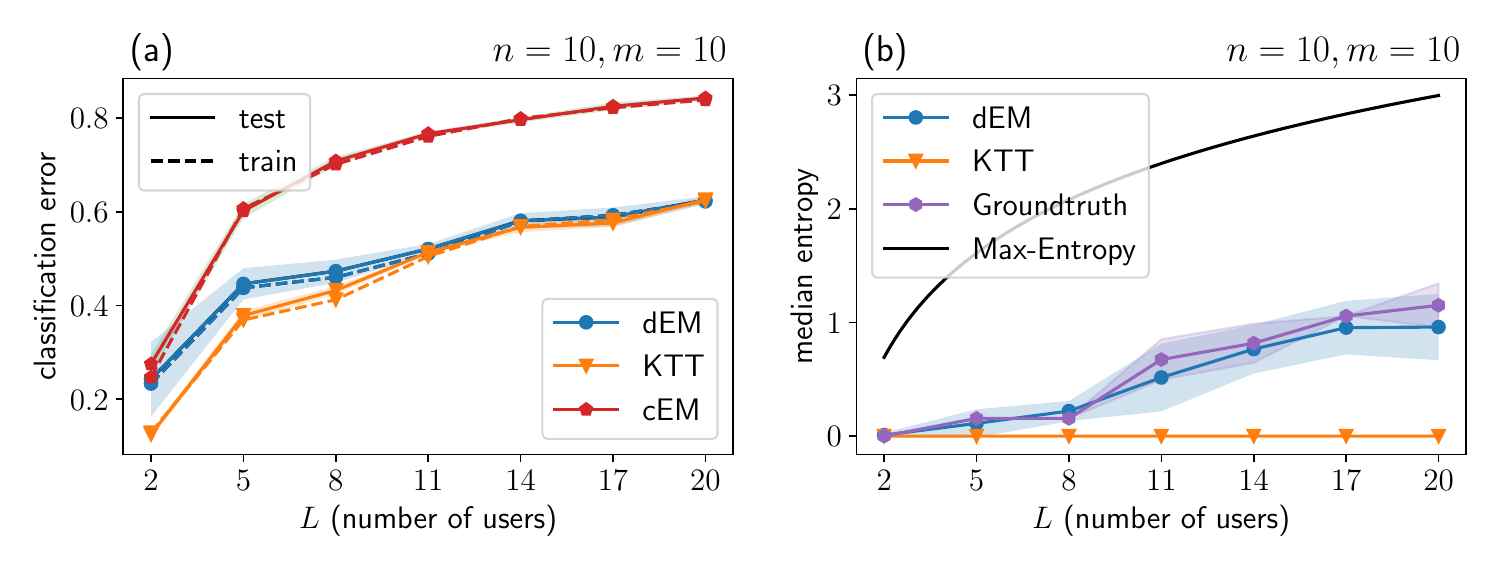}
    \figspace
    \caption{Classification error and assignment entropy on the Last.fm dataset.
     }
    \label{fig:last-fm}

    \bigskip

    \centering
    \includegraphics[width=0.5\linewidth]{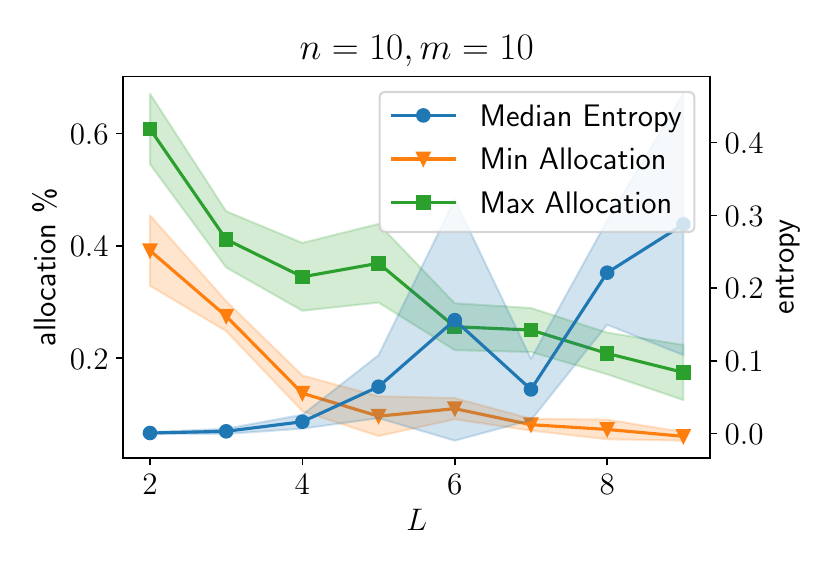}
    \figspace
    \caption{   \label{fig:real-lastfm} Allocation-percentage and median entropy for \methodem on Last.fm.}
\end{figure}

As the number of users in the dataset increases, the performance deteriorates due   to the soft-assignment having to decide among a greater number of potential chains, coupled with the reality that some users have similar listening habits.
In Figure~\ref{fig:real-lastfm},
we explore the learning of a
mixture encompassing $L$ chains with values ranging from 2 to 10 by centering our attention on 20 users. This is done to discern   typical user behaviors, also known as archetypical behaviors~\cite{cutler1994archetypal}. 
We graphically represent the median assignment
entropy and the minimum and
maximum of the assignment probabilities
over the chains for \methodem. In this context, the probability of assignment to a chain 
$\ell \in [L]$ is defined as
$\frac 1 r \sum_{\mathbf x \in \mathbf X} a(\mathbf x, \ell)$.
We notice that as the number of chains increases, the entropy also tends to rise. This suggests that several chains produce comparable likelihoods for the trails, leading to greater uncertainty in assignment. As previously noted in the main content, determining a strategy for selecting $L$ remains an open challenge for future work.

\begin{figure}
    \centering

\tikzstyle{court}=[black!20, line width=2pt,samples=100]
\tikzstyle{player}=[circle, inner sep=0pt, text width=16pt, align=center, draw=black!80, line width=1pt, fill=white]
\tikzstyle{pass}=[-latex, black]
\tikzstyle{basket}=[player, rectangle, inner sep=2pt, minimum height=13pt, text centered]
\tikzstyle{start}=[player,fill=none,draw=blue,line width=2.5pt]

\scalebox{0.68}{
\begin{tikzpicture}[scale=1.6]

    \draw[court, domain=0:180] plot ({min(1.7, max(-1.7, 2*cos(\x)))}, {2*sin(\x)});
    \draw[court] (-0.5,0) -- (-0.5,1.2) -- (0.5,1.2) -- (0.5,0);
    \draw[court, domain=0:180] plot ({0.3*cos(\x)}, {1.2 + 0.3*sin(\x)});
    \draw[court, dashed, domain=0:180] plot ({0.3*cos(\x)}, {1.2 - 0.3*sin(\x)});
    \draw[court] (-1.8,0) -- (1.8,0);

    \node[player] (PG) at (0,2) {PG};
    \node[player] (SG) at (-1.3,1.5) {SG};
    \node[player] (PF) at (-1.0,0.6) {PF};
    \node[player] (C) at (0.9,1.0) {C};
    \node[player] (SF) at (1.1,0.3) {SF};

    \node[basket] (miss) at (-0.4,0) {$\mathsf{miss}$};
    \node[basket] (score) at (0.4,0) {$\mathsf{hit}$};

    \node[label={[label distance=6pt]below:{3.2s}}] at (PF) {};
    \draw[pass,opacity=0.17421865316690174,line width=1.7421865316690175pt,red] (PF) to (miss);
    \draw[pass,opacity=0.3847985576814498,line width=3.8479855768144984pt] (PF) to (SG);
    \node[label={[label distance=6pt]below:{3.5s}}] at (SF) {};
    \draw[pass,opacity=0.2136728663275433,line width=2.136728663275433pt,red] (SF) to (miss);
    \draw[pass,opacity=0.25372907576075426,line width=2.5372907576075425pt] (SF) to (PF);
    \draw[pass,opacity=0.3377098028408584,line width=3.377098028408584pt] (SF) to (SG);
    \node[label={[label distance=6pt]above:{3.0s}}] at (PG) {};
    \draw[pass,opacity=0.23277201800810687,line width=2.3277201800810685pt,red] (PG) to (miss);
    \draw[pass,opacity=0.2074350416129804,line width=2.074350416129804pt,green] (PG) to (score);
    \draw[pass,opacity=0.2365322622400459,line width=2.365322622400459pt] (PG) to (PF);
    \node[label={[label distance=6pt]above:{1.8s}}] at (SG) {};
    \node[start] at (SG) {};
    \draw[pass,opacity=0.386305757980682,line width=3.86305757980682pt] (SG) to (PF);
    \draw[pass,opacity=0.26621745786886863,line width=2.6621745786886866pt] (SG) to (SF);
    \node[label={[label distance=6pt]above:{4.9s}}] at (C) {};
    \draw[pass,opacity=0.18551262560900317,line width=1.8551262560900317pt,red] (C) to (miss);
    \draw[pass,opacity=0.3408037952378343,line width=3.4080379523783426pt] (C) to (PF);
    \draw[pass,opacity=0.2812897982035745,line width=2.812897982035745pt] (C) to (SG);
    
\end{tikzpicture}}
\hfill
\scalebox{0.68}{
\begin{tikzpicture}[scale=1.6]

    \draw[court, domain=0:180] plot ({min(1.7, max(-1.7, 2*cos(\x)))}, {2*sin(\x)});
    \draw[court] (-0.5,0) -- (-0.5,1.2) -- (0.5,1.2) -- (0.5,0);
    \draw[court, domain=0:180] plot ({0.3*cos(\x)}, {1.2 + 0.3*sin(\x)});
    \draw[court, dashed, domain=0:180] plot ({0.3*cos(\x)}, {1.2 - 0.3*sin(\x)});
    \draw[court] (-1.8,0) -- (1.8,0);

    \node[player] (PG) at (0,2) {PG};
    \node[player] (SG) at (-1.3,1.5) {SG};
    \node[player] (PF) at (-1.0,0.6) {PF};
    \node[player] (C) at (0.9,1.0) {C};
    \node[player] (SF) at (1.1,0.3) {SF};

    \node[basket] (miss) at (-0.4,0) {$\mathsf{miss}$};
    \node[basket] (score) at (0.4,0) {$\mathsf{hit}$};

    \node[label={[label distance=6pt]below:{3.3s}}] at (PF) {};
    \draw[pass,opacity=0.8075762006944756,line width=5.075762006944755pt,red] (PF) to (miss);
    \draw[pass,opacity=0.19242379930552447,line width=1.9242379930552447pt,green] (PF) to (score);
    \node[label={[label distance=6pt]below:{2.3s}}] at (SF) {};
    \draw[pass,opacity=0.1927787148366214,line width=1.927787148366214pt,red] (SF) to (miss);
    \draw[pass,opacity=0.33217766424127865,line width=3.3217766424127864pt] (SF) to (SG);
    \node[label={[label distance=6pt]above:{3.0s}}] at (PG) {};
    \draw[pass,opacity=0.23758948295778162,line width=2.3758948295778164pt,red] (PG) to (miss);
    \draw[pass,opacity=0.20027838571496026,line width=2.002783857149603pt,green] (PG) to (score);
    \draw[pass,opacity=0.20319853585774264,line width=2.0319853585774266pt] (PG) to (SF);
    \node[label={[label distance=6pt]above:{1.7s}}] at (SG) {};
    \node[start] at (SG) {};
    \draw[pass,opacity=0.32372030104618044,line width=3.2372030104618044pt] (SG) to (SF);
    \draw[pass,opacity=0.2729473982795745,line width=2.729473982795745pt] (SG) to (C);
    \node[label={[label distance=6pt]above:{5.0s}}] at (C) {};
    \draw[pass,opacity=0.2133173246087413,line width=2.133173246087413pt,red] (C) to (miss);
    \draw[pass,opacity=0.3419116990103844,line width=3.4191169901038436pt] (C) to (SF);
    \draw[pass,opacity=0.21227502806208576,line width=2.1227502806208576pt] (C) to (SG);
    
\end{tikzpicture}}
\hfill
\scalebox{0.68}{
\begin{tikzpicture}[scale=1.6]

    \draw[court, domain=0:180] plot ({min(1.7, max(-1.7, 2*cos(\x)))}, {2*sin(\x)});
    \draw[court] (-0.5,0) -- (-0.5,1.2) -- (0.5,1.2) -- (0.5,0);
    \draw[court, domain=0:180] plot ({0.3*cos(\x)}, {1.2 + 0.3*sin(\x)});
    \draw[court, dashed, domain=0:180] plot ({0.3*cos(\x)}, {1.2 - 0.3*sin(\x)});
    \draw[court] (-1.8,0) -- (1.8,0);

    \node[player] (PG) at (0,2) {PG};
    \node[player] (SG) at (-1.3,1.5) {SG};
    \node[player] (PF) at (-1.0,0.6) {PF};
    \node[player] (C) at (0.9,1.0) {C};
    \node[player] (SF) at (1.1,0.3) {SF};

    \node[basket] (miss) at (-0.4,0) {$\mathsf{miss}$};
    \node[basket] (score) at (0.4,0) {$\mathsf{hit}$};

    \node[label={[label distance=6pt]below:{3.8s}}] at (PF) {};
    \draw[pass,opacity=0.37392593149577874,line width=3.7392593149577875pt,red] (PF) to (miss);
    \draw[pass,opacity=0.2284713616965366,line width=2.284713616965366pt,green] (PF) to (score);
    \draw[pass,opacity=0.22081377774483252,line width=2.208137777448325pt] (PF) to (SG);
    \node[label={[label distance=6pt]below:{1.9s}}] at (SF) {};
    \draw[pass,opacity=0.2367965485820656,line width=2.367965485820656pt] (SF) to (PF);
    \draw[pass,opacity=0.380614120208571,line width=3.80614120208571pt] (SF) to (SG);
    \node[label={[label distance=6pt]above:{1.5s}}] at (PG) {};
    \node[start] at (PG) {};
    \node[start] at (SG) {};
    \draw[pass,opacity=0.20338539187545587,line width=2.033853918754559pt] (PG) to (PF);
    \draw[pass,opacity=0.21459358852473412,line width=2.1459358852473414pt] (PG) to (SF);
    \draw[pass,opacity=0.309456980303618,line width=3.09456980303618pt] (PG) to (SG);
    \draw[pass,opacity=0.2487899076175299,line width=2.487899076175299pt] (PG) to (C);
    \node[label={[label distance=6pt]above:{3.1s}}] at (SG) {};
    \draw[pass,opacity=0.25829774855031873,line width=2.5829774855031875pt,red] (SG) to (miss);
    \draw[pass,opacity=0.16082902411395902,line width=1.6082902411395903pt,green] (SG) to (score);
    \draw[pass,opacity=0.20289307596242384,line width=2.0289307596242385pt] (SG) to (PF);
    \draw[pass,opacity=0.2972631598272379,line width=2.972631598272379pt] (SG) to (SF);
    \node[label={[label distance=6pt]above:{5.2s}}] at (C) {};
    \draw[pass,opacity=0.17982582680726972,line width=1.798258268072697pt,red] (C) to (miss);
    \draw[pass,opacity=0.22391385656422996,line width=2.2391385656422997pt] (C) to (PF);

\end{tikzpicture}}
\hfill
\scalebox{0.68}{
\begin{tikzpicture}[scale=1.6]

    \draw[court, domain=0:180] plot ({min(1.7, max(-1.7, 2*cos(\x)))}, {2*sin(\x)});
    \draw[court] (-0.5,0) -- (-0.5,1.2) -- (0.5,1.2) -- (0.5,0);
    \draw[court, domain=0:180] plot ({0.3*cos(\x)}, {1.2 + 0.3*sin(\x)});
    \draw[court, dashed, domain=0:180] plot ({0.3*cos(\x)}, {1.2 - 0.3*sin(\x)});
    \draw[court] (-1.8,0) -- (1.8,0);

    \node[player] (PG) at (0,2) {PG};
    \node[player] (SG) at (-1.3,1.5) {SG};
    \node[player] (PF) at (-1.0,0.6) {PF};
    \node[player] (C) at (0.9,1.0) {C};
    \node[player] (SF) at (1.1,0.3) {SF};

    \node[basket] (miss) at (-0.4,0) {$\mathsf{miss}$};
    \node[basket] (score) at (0.4,0) {$\mathsf{hit}$};

    \node[label={[label distance=6pt]below:{3.4s}}] at (PF) {};
    \draw[pass,opacity=0.2646707546813019,line width=2.646707546813019pt] (PF) to (SF);
    \draw[pass,opacity=0.3492253077767991,line width=3.492253077767991pt] (PF) to (PG);
    \node[label={[label distance=6pt]below:{2.5s}}] at (SF) {};
    \draw[pass,opacity=0.22297372532204007,line width=2.2297372532204007pt,red] (SF) to (miss);
    \draw[pass,opacity=0.273387913148029,line width=2.73387913148029pt] (SF) to (PF);
    \draw[pass,opacity=0.29273244640530244,line width=2.9273244640530245pt] (SF) to (PG);
    \node[label={[label distance=6pt]above:{1.9s}}] at (PG) {};
    \node[start] at (PG) {};
    \draw[pass,opacity=0.16402163312886964,line width=1.6402163312886964pt,red] (PG) to (miss);
    \draw[pass,opacity=0.26280332083317187,line width=2.628033208331719pt] (PG) to (PF);
    \draw[pass,opacity=0.24797016209535766,line width=2.4797016209535765pt] (PG) to (SF);
    \node[label={[label distance=6pt]above:{3.9s}}] at (SG) {};
    \draw[pass,opacity=0.226501197906821,line width=2.26501197906821pt,red] (SG) to (miss);
    \draw[pass,opacity=0.16440778468778028,line width=1.6440778468778028pt,green] (SG) to (score);
    \draw[pass,opacity=0.21130649181278355,line width=2.1130649181278356pt] (SG) to (PF);
    \draw[pass,opacity=0.24279992439121978,line width=2.427999243912198pt] (SG) to (PG);
    \node[label={[label distance=6pt]above:{1.8s}}] at (C) {};
    \draw[pass,opacity=0.32645372749091417,line width=3.2645372749091415pt] (C) to (SF);
    \draw[pass,opacity=0.2172921815305613,line width=2.172921815305613pt] (C) to (PG);
    \draw[pass,opacity=0.26708452449449693,line width=2.670845244944969pt] (C) to (SG);

\end{tikzpicture}}
    \caption{
    Offensive strategies for the Golden State Warriors represented as a  mixture of $L=4$ CTMCs. We use a discretization rate of $\tau=0.1$. The scoring probabilities are (from left to right): $44\%$, $37\%$,
    $41\%$, and $41\%$. }


    \label{fig:nba-strategy}
    \label{fig:nba-strategy2}

\scalebox{0.68}{
\begin{tikzpicture}[scale=1.6]

    \draw[court, domain=0:180] plot ({min(1.7, max(-1.7, 2*cos(\x)))}, {2*sin(\x)});
    \draw[court] (-0.5,0) -- (-0.5,1.2) -- (0.5,1.2) -- (0.5,0);
    \draw[court, domain=0:180] plot ({0.3*cos(\x)}, {1.2 + 0.3*sin(\x)});
    \draw[court, dashed, domain=0:180] plot ({0.3*cos(\x)}, {1.2 - 0.3*sin(\x)});
    \draw[court] (-1.8,0) -- (1.8,0);

    \node[player] (PG) at (0,2) {PG};
    \node[player] (SG) at (-1.3,1.5) {SG};
    \node[player] (PF) at (-1.0,0.6) {PF};
    \node[player] (C) at (0.9,1.0) {C};
    \node[player] (SF) at (1.1,0.3) {SF};

    \node[basket] (miss) at (-0.4,0) {$\mathsf{miss}$};
    \node[basket] (score) at (0.4,0) {$\mathsf{hit}$};

    \node[label={[label distance=6pt]above:{2.4s}}] at (PG) {};
    \draw[pass,opacity=0.5490022121149478,line width=5.490022121149478pt,red] (PG) to (miss);
    \draw[pass,opacity=0.3042228275350427,line width=3.042228275350427pt,green] (PG) to (score);
    \node[label={[label distance=6pt]above:{1.5s}}] at (SG) {};
    \node[start] at (SG) {};
    \draw[pass,opacity=0.17239181907233098,line width=1.7239181907233099pt,red] (SG) to (miss);
    \draw[pass,opacity=0.3507830598589462,line width=3.507830598589462pt] (SG) to (SF);
    \draw[pass,opacity=0.27112934503052466,line width=2.711293450305247pt] (SG) to (PF);
    \node[label={[label distance=6pt]above:{1.4s}}] at (C) {};
    \draw[pass,opacity=0.5607077173097216,line width=5.607077173097216pt] (C) to (SG);
    \draw[pass,opacity=0.21548858185307543,line width=2.1548858185307544pt] (C) to (PF);
    \node[label={[label distance=6pt]below:{2.4s}}] at (SF) {};
    \draw[pass,opacity=0.2562511785994422,line width=2.5625117859944218pt,red] (SF) to (miss);
    \draw[pass,opacity=0.15415140871021468,line width=1.5415140871021469pt,green] (SF) to (score);
    \draw[pass,opacity=0.21956172174880614,line width=2.1956172174880613pt] (SF) to (SG);
    \draw[pass,opacity=0.35706462584741144,line width=3.5706462584741145pt] (SF) to (PF);
    \node[label={[label distance=6pt]below:{3.2s}}] at (PF) {};
    \draw[pass,opacity=0.33126838473279757,line width=3.3126838473279756pt,red] (PF) to (miss);
    \draw[pass,opacity=0.18751661934018496,line width=1.8751661934018495pt,green] (PF) to (score);
    \draw[pass,opacity=0.2067510439006554,line width=2.067510439006554pt] (PF) to (SG);

\end{tikzpicture}}
\hfill
\scalebox{0.68}{
\begin{tikzpicture}[scale=1.6]

    \draw[court, domain=0:180] plot ({min(1.7, max(-1.7, 2*cos(\x)))}, {2*sin(\x)});
    \draw[court] (-0.5,0) -- (-0.5,1.2) -- (0.5,1.2) -- (0.5,0);
    \draw[court, domain=0:180] plot ({0.3*cos(\x)}, {1.2 + 0.3*sin(\x)});
    \draw[court, dashed, domain=0:180] plot ({0.3*cos(\x)}, {1.2 - 0.3*sin(\x)});
    \draw[court] (-1.8,0) -- (1.8,0);

    \node[player] (PG) at (0,2) {PG};
    \node[player] (SG) at (-1.3,1.5) {SG};
    \node[player] (PF) at (-1.0,0.6) {PF};
    \node[player] (C) at (0.9,1.0) {C};
    \node[player] (SF) at (1.1,0.3) {SF};

    \node[basket] (miss) at (-0.4,0) {$\mathsf{miss}$};
    \node[basket] (score) at (0.4,0) {$\mathsf{hit}$};

    \node[label={[label distance=6pt]above:{1.7s}}] at (PG) {};
    \draw[pass,opacity=0.2549082518489791,line width=2.549082518489791pt,red] (PG) to (miss);
    \draw[pass,opacity=0.27712076105491706,line width=2.7712076105491708pt,green] (PG) to (score);
    \draw[pass,opacity=0.3247724541152446,line width=3.247724541152446pt] (PG) to (PF);
    \node[label={[label distance=6pt]above:{1.7s}}] at (SG) {};
    \draw[pass,opacity=0.28338471699693496,line width=2.8338471699693497pt,red] (SG) to (miss);
    \draw[pass,opacity=0.1957239411713233,line width=1.957239411713233pt,green] (SG) to (score);
    \draw[pass,opacity=0.36069075818733076,line width=3.6069075818733074pt] (SG) to (PF);
    \node[label={[label distance=6pt]above:{11.1s}}] at (C) {};
    \draw[pass,opacity=0.46079547376975943,line width=4.607954737697594pt] (C) to (SG);
    \draw[pass,opacity=0.26040210012986276,line width=2.6040210012986273pt] (C) to (PF);
    \node[label={[label distance=6pt]below:{2.4s}}] at (SF) {};
    \draw[pass,opacity=0.35024026610497166,line width=3.5024026610497168pt,red] (SF) to (miss);
    \draw[pass,opacity=0.2615065743578059,line width=2.615065743578059pt] (SF) to (SG);
    \draw[pass,opacity=0.20632520237056867,line width=2.0632520237056866pt] (SF) to (PF);
    \node[label={[label distance=6pt]below:{2.3s}}] at (PF) {};
    \node[start] at (PF) {};
    \draw[pass,opacity=0.15349453788466144,line width=1.5349453788466145pt,red] (PF) to (miss);
    \draw[pass,opacity=0.5915056472219498,line width=5.915056472219497pt] (PF) to (SG);

\end{tikzpicture}}
\hfill
\scalebox{0.68}{
\begin{tikzpicture}[scale=1.6]

    \draw[court, domain=0:180] plot ({min(1.7, max(-1.7, 2*cos(\x)))}, {2*sin(\x)});
    \draw[court] (-0.5,0) -- (-0.5,1.2) -- (0.5,1.2) -- (0.5,0);
    \draw[court, domain=0:180] plot ({0.3*cos(\x)}, {1.2 + 0.3*sin(\x)});
    \draw[court, dashed, domain=0:180] plot ({0.3*cos(\x)}, {1.2 - 0.3*sin(\x)});
    \draw[court] (-1.8,0) -- (1.8,0);

    \node[player] (PG) at (0,2) {PG};
    \node[player] (SG) at (-1.3,1.5) {SG};
    \node[player] (PF) at (-1.0,0.6) {PF};
    \node[player] (C) at (0.9,1.0) {C};
    \node[player] (SF) at (1.1,0.3) {SF};

    \node[basket] (miss) at (-0.4,0) {$\mathsf{miss}$};
    \node[basket] (score) at (0.4,0) {$\mathsf{hit}$};

    \node[label={[label distance=6pt]above:{1.5s}}] at (PG) {};
    \draw[pass,opacity=0.30352284776261124,line width=3.0352284776261125pt] (PG) to (SG);
    \draw[pass,opacity=0.4649132504483589,line width=4.649132504483589pt] (PG) to (PF);
    \node[label={[label distance=6pt]above:{1.5s}}] at (SG) {};
    \node[start] at (SG) {};
    \draw[pass,opacity=0.2910253566899712,line width=2.910253566899712pt,red] (SG) to (miss);
    \draw[pass,opacity=0.37575375598237404,line width=3.7575375598237404pt] (SG) to (PF);
    \node[label={[label distance=6pt]above:{5.6s}}] at (C) {};
    \draw[pass,opacity=0.22141734298212745,line width=2.2141734298212743pt,red] (C) to (miss);
    \draw[pass,opacity=0.23431454033192414,line width=2.3431454033192414pt] (C) to (SG);
    \draw[pass,opacity=0.27242143558245535,line width=2.7242143558245537pt] (C) to (PF);
    \node[label={[label distance=6pt]below:{2.2s}}] at (SF) {};
    \draw[pass,opacity=0.5266853161972386,line width=5.266853161972386pt] (SF) to (SG);
    \node[label={[label distance=6pt]below:{3.6s}}] at (PF) {};
    \draw[pass,opacity=0.27427168209987246,line width=2.7427168209987247pt,red] (PF) to (miss);
    \draw[pass,opacity=0.1534552912922337,line width=1.534552912922337pt,green] (PF) to (score);
    \draw[pass,opacity=0.5186571651559817,line width=5.186571651559817pt] (PF) to (SG);

\end{tikzpicture}}
\hfill
\scalebox{0.68}{
\begin{tikzpicture}[scale=1.6]

    \draw[court, domain=0:180] plot ({min(1.7, max(-1.7, 2*cos(\x)))}, {2*sin(\x)});
    \draw[court] (-0.5,0) -- (-0.5,1.2) -- (0.5,1.2) -- (0.5,0);
    \draw[court, domain=0:180] plot ({0.3*cos(\x)}, {1.2 + 0.3*sin(\x)});
    \draw[court, dashed, domain=0:180] plot ({0.3*cos(\x)}, {1.2 - 0.3*sin(\x)});
    \draw[court] (-1.8,0) -- (1.8,0);

    \node[player] (PG) at (0,2) {PG};
    \node[player] (SG) at (-1.3,1.5) {SG};
    \node[player] (PF) at (-1.0,0.6) {PF};
    \node[player] (C) at (0.9,1.0) {C};
    \node[player] (SF) at (1.1,0.3) {SF};

    \node[basket] (miss) at (-0.4,0) {$\mathsf{miss}$};
    \node[basket] (score) at (0.4,0) {$\mathsf{hit}$};

    \node[label={[label distance=6pt]above:{1.8s}}] at (PG) {};
    \node[start] at (PG) {};
    \draw[pass,opacity=0.18616675203452066,line width=1.8616675203452067pt,red] (PG) to (miss);
    \draw[pass,opacity=0.30391780876473257,line width=3.0391780876473256pt] (PG) to (SG);
    \draw[pass,opacity=0.24310310403748392,line width=2.4310310403748394pt] (PG) to (PF);
    \node[label={[label distance=6pt]above:{2.5s}}] at (SG) {};
    \draw[pass,opacity=0.20904850708227513,line width=2.0904850708227514pt,red] (SG) to (miss);
    \draw[pass,opacity=0.4104426518500701,line width=4.104426518500701pt] (SG) to (PG);
    \node[label={[label distance=6pt]above:{7.6s}}] at (C) {};
    \draw[pass,opacity=0.33479783678598196,line width=3.3479783678598194pt] (C) to (PG);
    \draw[pass,opacity=0.33512439437772085,line width=3.3512439437772086pt] (C) to (PF);
    \node[label={[label distance=6pt]below:{3.3s}}] at (SF) {};
    \draw[pass,opacity=0.29383515431545587,line width=2.9383515431545586pt,red] (SF) to (miss);
    \draw[pass,opacity=0.1812620348932682,line width=1.812620348932682pt,green] (SF) to (score);
    \draw[pass,opacity=0.35049065814899677,line width=3.504906581489968pt] (SF) to (PF);
    \node[label={[label distance=6pt]below:{3.4s}}] at (PF) {};
    \draw[pass,opacity=0.45880765181475996,line width=4.5880765181475995pt] (PF) to (PG);

\end{tikzpicture}}

    \caption{ Strategies for the New York Knicks, anaogous to Figure~\ref{fig:nba-strategy}. The scoring probabilities are (from left to right): $37\%$, $36\%$,
    $35\%$, and $40\%$.}
    
    \label{fig:nba-strategy3}
    
\end{figure}

\spara{Markovletics: Navigating Sports Strategies through CTMCs}
For every team in the 2022 season, we utilize trails crafted from an NBA passing dataset to deduce a blend of CTMCs. We aim for each chain to represent an offensive team's tactic, shedding light on the probability of point-scoring associated with that particular strategy by inspecting its steady-state distribution.
In Figure~\ref{fig:nba-strategy}, we highlight two of the offensive tactics of Golden State Warriors (GSW) discerned from \methodem with four CTMCs. The basketball game involves five positions: Point Guard (PG), Shooting Guard (SG), Power Forward (PF), Center (C), and Small Forward (SF). Each position is annotated with the calculated ball holding time. Arrows represent potential passes between positions, with their thickness and opacity denoting the pass's probability. Passes with a low likelihood (less than 0.2) are excluded for visual convenience. The player with the highest starting probability is highlighted in blue. Shoot attempts are highlighted in red (miss) and green (hit). Each strategy illuminates unique offensive patterns. 
We show another set of 4 offensive strategies learned from trails of the New York Knicks in the 2022 season in Figure~\ref{fig:nba-strategy3}.

We also provide additional qualitative results on the NBA dataset.
Since \hit and \miss are the sole absorbing states, $\pi_K(\hit)$ and $\pi_K(\miss) = 1 - \pi_K(\hit)$, they signify the odds of scoring or not scoring points, respectively. Thus, we can ascribe a score likelihood to each tactic.
To gauge the efficacy of the deduced mixture, we assess its predictive precision, as done in other work \cite{cervone2014pointwise}. With a given mixture of continuous-time Markov chains and a trail prefix $x' = (x_t)_{0 \le t \le t'}$ halting prior to reaching the absorbing states \hit or \miss at time $t'$, we ascertain the probabilities $\Pr[x' \cap \ell \mid \mathbf K]$ and $\Pr[x_\infty = \hit \mid x' \cap K^\ell]$ for every $\ell \in [L]$, facilitating the determination of the score likelihood $\Pr[x_\infty = \hit]$ via the theorem of total probability.
We use  trails with a 80\%-20\% train-test split from the teams
Golden State Warriors (GSW), Boston Celtics (BOS), Los Angeles Lakers (LAL), Miami Heat (MIA), Los Angeles Clippers (LAC), and Houston Rockets (HOU) in the 2022 season.
We plot the predictive
accuracy of the chains learned
via \methodem, \methodkausik,
and \methodemcont. As a baseline,
we implemented a recurrent neural network
using Pytorch~\cite{goodfellow2016deep} that is 
trained on the set of discretized trails.

\begin{figure}
    \centering
    \includegraphics[width=0.7\linewidth]{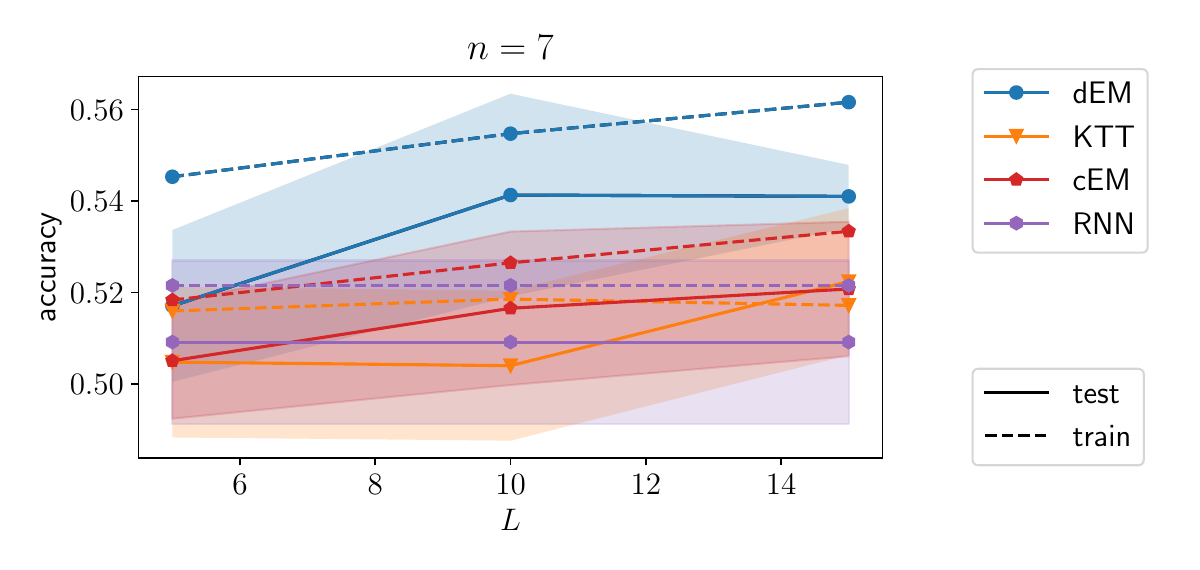}
    \figspace
    \caption{   \label{fig:real-nba} Train and test accuracy of \miss and \hit prediction using \methodem, \methodemcont, \methodkausik and RNNs on the NBA dataset.}
\end{figure}

\section{Conclusion}
\label{sec:concl}
This research delved extensively into the study of learning mixtures of CTMCs, presenting novel algorithms and conducting comparisons with leading competitors across synthetic and real-world scenarios.   Our methods have been proven effective in real-world scenarios, as seen in the \textit{Last.fm} application. Additionally, we introduced the innovative concept of  {\it Markovletics}  for learning offensive tactics in NBA. In essence, our research   adds a fresh dimension to the theoretical aspects of Markov chains and exemplifies its real-world applicability.  Our study raises several intriguing questions, including the choice of the parameter $L$ and the expansion of our techniques to a broader spectrum of datasets, including those from bioinformatics~\cite{beerenwinkel2009markov}.

\bibliographystyle{alpha}
\bibliography{ref}

\newpage
\appendix


\def\Yd{Y^{\mathrm{d}}}
\def\Yc{Y^{\mathrm{c}}}

\section{Effect of the discretization parameter $\tau$} 
\label{sec:appendix-tau}

We remind the reader of the following formal
definition of a CTMC with rate matrix
$K \in \R^{n \times n}$.
Let $M \in \R^{n \times n}$ be the transition
matrix of a discrete-time Markov chain given
by $M_{yz} \coloneqq K_{xy} / |K_{yy}|$
for $y \not= z$ and $M_{yy} = 0$. Let
$\Yd(i) \in [n]$ be the state of a
random walk through $M$ at step $i \in \N_0$.
To define the continuous-time process,
we sample transition times
$T(i) \sim \mathrm{Exp}(|K_{yy}|)$
where $y = \Yd(i)$.
Now, for a time $t \ge 0$,
we set $\Yc(t) = \Yd(i)$ where 
$i \in \N_0$ is such that
$\sum_{j < i} T(j) \le t < \sum_{j \le i} T(j)$.

In the remainder of this section, we
provide further intuition on the quality of estimation of $\mathbf K$
for different values of $\tau$, as discussed in the
introduction and Section~\ref{subsec:discretization}.
  
Figure~\ref{fig:discret-problems} demonstrates the significance of selecting an appropriate value for $\tau$. The illustration depicts transitions from state 1 to state 2 in two separate chains, each with distinct transition rates of 1 and 10, respectively. The choice of the time scale $\tau$ has a substantial impact on the observed transition probabilities. It is essential to carefully choose $\tau$ (e.g., $\tau=0.1$) in order to distinguish between the discretized chains, as we discussed in Section~\ref{subsec:discretization}. The underlying issue is that when one chain within the mixture transitions significantly faster than another, using the correct discretization choice becomes critical. A too-small discretization may result in too few observed transitions in the slower chain, while a too-large discretization may cause both chains to converge to their potentially identical stationary distributions, hindering the differentiation of the trajectories.

\begin{figure}[htp]
    \centering
    \begin{tikzpicture}[state/.style={circle, draw=black, thick, minimum size=15pt, inner sep=0pt}, scale=0.8, every node/.style={transform shape}]  
        \node[align=right,anchor=east] at (-1, 0.5) {CTMCs $K^1$ and $K^2$};
        \node[align=right,anchor=east] at (-3, -2) {discretized chains \\[0.5mm] $e^{K^1 \tau}$ and $e^{K^2 \tau}$};
        
        \node[align=right,anchor=east] at (-1, -1) {$\tau=5$};
        \node[align=right,anchor=east] at (-1, -2) {$\tau=0.1$};
        \node[align=right,anchor=east] at (-1, -3) {$\tau=0.002$};

        \node[state] (x1) at (0, 0.5) {1};
        \node[state] (y1) at (1.5, 0.5) {2};
        \draw[->,thick] (x1) -- (y1) node[midway,above] {$1$};
        
        \node[state] (x2) at (3, 0.5) {1};
        \node[state] (y2) at (4.5, 0.5) {2};
        \draw[->,thick] (x2) -- (y2) node[midway,above] {$10$};

        \node[state] (x3) at (0, -1) {1};
        \node[state] (y3) at (1.5, -1) {2};
        \draw[->,thick] (x3) -- (y3) node[pos=0.45,above] {\small $\approx\! 0.99$};
        
        \node[state] (x4) at (3, -1) {1};
        \node[state] (y4) at (4.5, -1) {2};
        \draw[->,thick] (x4) -- (y4) node[pos=0.45,above] {\small $\approx\! 1.00$};

        \node[state] (x5) at (0, -2) {1};
        \node[state] (y5) at (1.5, -2) {2};
        \draw[->,thick] (x5) -- (y5) node[pos=0.45,above] {\small $\approx\! 0.63$};
        
        \node[state] (x6) at (3, -2) {1};
        \node[state] (y6) at (4.5, -2) {2};
        \draw[->,thick] (x6) -- (y6) node[pos=0.45,above] {\small $\approx\! 0.10$};

        \node[state] (x7) at (0, -3) {1};
        \node[state] (y7) at (1.5, -3) {2};
        \draw[->,thick] (x7) -- (y7) node[pos=0.45,above] {\small $\approx\! 0.02$};
        
        \node[state] (x8) at (3, -3) {1};
        \node[state] (y8) at (4.5, -3) {2};
        \draw[->,thick] (x8) -- (y8) node[pos=0.45,above] {\small $\approx\! 0.00$};
        
        \draw [decorate,decoration={brace,amplitude=5pt},xshift=-4pt,yshift=0pt]
(-2.4, -3.25) -- (-2.4, -0.75) node [black,midway,xshift=9pt] {};
    \end{tikzpicture}

    \caption{Choosing the right value of $\tau$ is crucial.}
    \label{fig:discret-problems}
\end{figure}
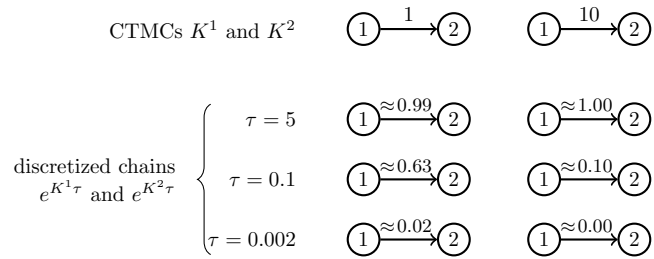

Figure~\ref{fig:choosing-tau} shows the asymptotic trends of the variables in the estimation of $\mathbf K$  with respect to $\tau$ that we introduced in Section~\ref{subsec:discretization}:
the total observation time for holdings periods without transitions, the count of observed transitions and the tally of bad transitions. Owing to the varying gradient as $\tau$ approaches 0, we can adjust $\tau$ to yield a minimal set of bad transitions while ensuring a significant number of quality  (i.e., not bad) transitions.

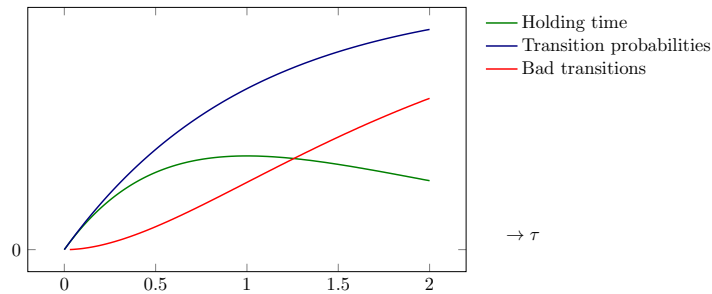
\begin{figure}[htp]
    \centering
    \scalebox{0.7}{\begin{tikzpicture}
        \begin{axis}[domain=0:2, restrict y to domain=0:10, width=0.6\linewidth, legend pos=outer north east, legend cell align=left, height=0.4\linewidth, legend style={fill opacity=0.5, text opacity=1, draw=none}, ytick={0}, xlabel={$\to \tau$}, x label style={at={(axis description cs: 1.13, 0.28)}}]
            \addplot[samples=200, smooth, thick, green!50!black] {x * exp(-1 * x)};
            \addlegendentry{ Holding time}
            \addplot[samples=200, smooth, thick, blue!50!black] {1 - exp(-x)};
            \addlegendentry{ Transition probabilities}
            \addplot[samples=200, smooth, thick, red] {1 - (1 + x) * exp(-x)};
            \addlegendentry{ Bad transitions}
        \end{axis}
    \end{tikzpicture}}
    \caption{Estimating $\mathbf K$.}
    \label{fig:choosing-tau}
\end{figure}

\section{Proofs from Section~\ref{subsec:discretization}}
\label{sec:appendix-sc}
We define $q_{y}^{\ell} \coloneqq e^{K_{yy}^{\ell}\tau}=\Pr\left[E \ge\tau\right]$
for $E \sim \mathrm Exp(|K_{yy}|)$
as the probability to remain in state $y$ in time $\tau$. Since
we are observing states at discrete time intervals of size $\tau$,
our estimator is 
\[
\hat{q}_y^{\ell} = \frac{1}{c_{y}^{\ell}} \left|\left\{ \mathbf{y}\in{\bf X}^{\ell},0 \le i < m:\mathbf{x}_{i}=y\land\mathbf{x}_{i+1}=y\right\} \right|
\]
where
$
    c_{y}^{\ell} = \left|\left\{ \mathbf{x}\in{\bf X}^{\ell},i:\mathbf{x}_{i}=y\right\} \right|
$
is the number of transitions from chain $\ell$ that traverse through $y$.
Note that $\hat{q}_{y}^{\ell}$ is a biased estimator of $q_{y}^{\ell}$
because we are unable to tell whether transitions are bad (see Definition~\ref{dfn:bad}). That is, when estimating the holding
time with our estimator, it could happen
that $\mathbf{x}_{i}=x_{\tau i}=y$ and $\mathbf{x}_{i+1}={x}_{\tau(i+1)}=y$
but there is a $\xi\in(0, 1)$ with ${x}_{\tau i + \xi}\not=y$. 
We thus have to set $\tau$ small enough to avoid bad samples. In
particular, we set
\[
\tau = \frac{\sqrt{\epsilon_{\mathrm h}}}{3K_{\max}}.
\]
We now estimate one row of the rate matrix $K_{y}^{\ell}=(K_{yz}^{\ell})_z$
for a fixed chain $\ell$ and state $y$. For brevity, we will omit $\ell$
in $K_{yz}=K_{yz}^{\ell}$ and omit $z$ in $q=q_{z}^{\ell}$,
$\hat{q}=\hat{q}_{y}^{\ell}$, and $m=m_{y}^{\ell}$.

\smallskip
{\sc Lemma}~\ref{lem:holding}.
\emph{
Given $c=\Omega\left(e^{\left|K_{yy}\right|\tau}\epsilon_{\mathrm h}^{-2}\log(Ln)\right)=\Omega\left(\epsilon_{\mathrm h}^{-2}\log(Ln)\right)$
transitions, our estimator $\hat{q}$ for the holding time satisfies
\[
\left|\hat{q}-q\right|\le\epsilon_{\mathrm h}q
\]
with high probability.
}

\begin{proof}
We first bound the number of bad transitions. In particular, we want that
only a $2\gamma \coloneqq \frac{\epsilon_{\mathrm h}}{2}q$ fraction of the $c$ transitions
are bad so that we do not incur to much error from these samples.
We first calculate the probability to obtain a single bad sample.
Here, we use that the larger $\left|K_{yy}\right|$, the more likely
it is to switch states, so the probability of obtaining a bad sample
is maximized in the state with largest $\left|K_{yy}\right|$. Let
thus $E, E' \sim\mathrm{Exp}\left(K_{\max}\right)$ be independent random
variables. By the memorylessness of the exponential distribution,
\begin{align*}
\Pr\left[E + E'<\tau\right] & =\int_{0}^{\tau}\Pr\left[E'\le\tau-t\right]f_{E}(t)dt\\
 & =K_{\max}\int_{0}^{\tau}\left(1-e^{-K_{\max}(\tau-t)}\right)e^{-K_{\max}t}dt\\
 & =1-\left(1+K_{\max}\tau\right)e^{-K_{\max}\tau}\\
 & \le1-\left(1+K_{\max}\tau\right)\left(1-K_{\max}\tau\right)\\
 & =K_{\max}^{2}\tau^{2} 
\end{align*}
where $f_E$ is the PDF of $E$.
We further bound
\begin{multline*}
K_{\max}^{2}\tau^{2}=
\frac{\epsilon_{\mathrm h}}{9}
\le\frac{\epsilon_{\mathrm h}}{4}-\frac{\epsilon_{\mathrm h}}{12}\cdot\frac{\sqrt{\epsilon_{\mathrm h}}\left|K_{xx}\right|}{K_{\max}}
\le\frac{\epsilon_{\mathrm h}}{4}\left(1-\frac{1}{3}\sqrt{\epsilon_{\mathrm h}}\frac{\left|K_{xx}\right|}{K_{\max}}\right) \\
\le\frac{\epsilon_{\mathrm h}}{4}e^{\frac{1}{3}\sqrt{\epsilon_{\mathrm h}}\frac{K_{xx}}{K_{\max}}}
=\frac{\epsilon_{\mathrm h}}{4}e^{K_{xx}\tau}=\frac{\epsilon_{\mathrm h}}{4}q=\gamma .
\end{multline*}
Let now $H_{j}=1$ if there is no state change in the observed transition $j \in [c]$.
Let $B_{j}=1$ if the $j$-th transition is bad. By a Chernoff bound,
\begin{align*}
\Pr\left[\sum\nolimits_{j=1}^{c}B_{j}\ge2c\gamma\right]\le\exp\left(-\frac{1}{3}c\gamma\right)
=\exp\left(-\frac{1}{12}\epsilon_{\mathrm h}cq\right)=O\left(\frac{1}{\mathrm{poly}\left(Ln\right)}\right).
\end{align*}
Thus, with high probability, at most a $2\gamma$ fraction of the
samples is bad. We therefore incur an additive error of at most $2\gamma=\frac{\epsilon_{\mathrm h}}{2}q$
from bad samples. Since $2\gamma\le\frac{1}{2}$, we can use another
Chernoff bound on the good samples, which make up at least half the
total samples, to show that the estimation error from sampling is
at most $\frac{\epsilon_{\mathrm h}}{2}$ with high probability:
\begin{align*}
\Pr\left[\Big|q-\frac{2}{c}\sum\nolimits_{j=1}^{c/2}H_{j}\Big|\ge\frac{\epsilon_{\mathrm h}}{2}q\right]
 \le2\exp\left(-\frac{1}{24}\epsilon_{\mathrm h}^{2}cq\right)
 =O\left(\frac{1}{\mathrm{poly}\left(Ln\right)}\right). \qedhere
\end{align*}
\end{proof}
Let $p_{z}\coloneqq\frac{K_{yz}}{\left|K_{yy}\right|}$ be the transition
probability from state $y$ to $z$ within time $\tau$. We estimate
$p_{z}$ through 
\[
\hat{p}_{z} \coloneqq \frac{\left|\left\{ \mathbf{x}\in{\bf X}^{\ell},0 \le i < m:\mathbf{x}_{i}=y\land\mathbf{x}_{i+1}=z\right\} \right|}{\left|\left\{ \mathbf{x}\in{\bf X}^{\ell},0 \le i < m:\mathbf{x}_{i}=y\land{\bf x}_{i+1}\not=y\right\} \right|}
\]
for all states $y\not=x$. Let also
$\mathbf{p}\coloneqq\left(p_{z}\right)_{z\not=y}$ and $\hat{{\bf p}} \coloneqq \left(\hat{p}_{z}\right)_{z\not=y}$
be discrete probability vectors.

\begin{lemma}
\label{lem:transition}
With $c=\Omega\Big(\frac{n\kappa}{\epsilon_{\mathrm t}^{2}\sqrt{\epsilon_{\mathrm h}}}\log\left(Ln\right)\Big)$
transitions, our estimator for the transition probabilities
satisfies
\[
    \mathrm{TV}({\bf p},\hat{{\bf p}})\le\epsilon_{\mathrm t}
\]
with high probability.
\end{lemma}

\begin{proof}
We first bound the probability to see a transition in a given sample.
Since the probability of observing a transition is minimized for $K_{\min}$,
let $E=\mathrm{Exp}\left(K_{\min}\right)$. Using the fact that $e^{-x} \geq 1-\frac{x}{2}$ for $x \in [0, 1.59]$, we can   bound the
probability of observing a transition as
\[
\Pr\left[E<\tau\right] =1-e^{-K_{\min}\tau}
 =1-e^{-\kappa^{-1}\frac{\sqrt{\epsilon_{\mathrm h}}}{3}}
 \ge \frac{\sqrt{\epsilon_{\mathrm h}}}{6\kappa}.
\]
We obtain that at least a $\frac{\sqrt{\epsilon_{\mathrm h}}}{12\kappa}$
fraction of the samples are transitions with high probability, since
by a Chernoff bound,
\begin{align*}
\Pr\left[\frac{1}{m}\sum\nolimits_{i=1}^{c}T_{j}\le\frac{\sqrt{\epsilon_{\mathrm h}}}{12\kappa}\right]
    =\Pr\left[\sum\nolimits_{i=1}^{c}T_{j}\le\left(1-\frac{1}{2}\right)c\frac{\sqrt{\epsilon_{\mathrm h}}}{6\kappa}\right]
    \le\exp\left(-c\frac{\sqrt{\epsilon_{\mathrm h}}}{\kappa}\cdot\frac{1}{48}\right)
    =O\left(\frac{1}{\mathrm{poly}\left(Ln\right)}\right).
\end{align*}
where $T_j = 1 - H_j$ is a random variable
for observing a state change in the $j$-th transition.
With a similar argument as in Lemma~\ref{lem:holding}, we can show that
if $\kappa\sqrt{\epsilon_{\mathrm h}}\le\frac{9}{24}\epsilon_{\mathrm t}$, only a
$\frac{\epsilon_{\mathrm t}}{2}$-fraction of the transition samples are bad.
Furthermore, it is well know that to estimate a discrete distribution
with support $n-1$ to $\mathrm{TV}\left(\mathbf{p},\hat{\mathbf{p}}\right)\le\frac{\epsilon_{\mathrm t}}{2}$,
we require $\Omega(n/\epsilon_{\mathrm t}^{2})$ many samples of
transitions.
This requires that $\frac{\sqrt{\epsilon_{\mathrm h}}}{12\kappa}c=\Omega(n / \epsilon_{\mathrm t}^{2})$
which is satisfied by setting $c$ as in the theorem statement.
\end{proof}
Since $e^{K_{yy}\tau}=q$, we set
\[
\hat{K}_{yy} \coloneqq \frac{\log\left(\hat{q}\right)}{\tau}
\qquad
\textrm{and}
\qquad
\hat{K}_{yz} \coloneqq \hat{p}_{z}\left|\hat{K}_{yy}\right|.
\]
We define the $y$-th row of the $\ell$-th chain
as $K_{y}^{\ell}=(K_{yz})_{z}$ and our estimate as
$\hat{K}_{y}^{\ell}=(\hat{K}_{yz})_{z}$.

\smallskip
Theorem~\ref{thm:sc}
\emph{
Using $c=\Omega\left(\frac{\kappa^{2}}{\epsilon^{3}}\left(n+\frac{\kappa^{2}}{\epsilon}\right)\log\left(Ln\right)\right)$
transitions, we can estimate $K_{y}^{\ell}$ such that $\mathrm{TV}({ K}_{y}^{\ell},\hat{{K}}_{y}^{\ell})\le\epsilon$
with high probability.
}


\begin{proof}
By the definition of the $\mathrm{TV}$ distance,
\begin{align*}
\mathrm{TV}({K}_{y}^{\ell},\hat{{K}}_{y}^{\ell}) & =\frac{1}{2}\int_{0}^{\infty}\sum\nolimits_{z\not=y}\left|\frac{\hat{K}_{yz}}{\hat{K}_{yy}}\hat{K}_{yy}e^{\hat{K}_{yy}t}-\frac{K_{yz}}{K_{yy}}K_{yy}e^{K_{yy}t}\right|dt\\
 & =\frac{1}{2 \tau}\int_{0}^{\infty}\sum\nolimits_{z\not=y}\left|\hat{p}_{z}\log\left(\hat{q}\right)\hat{q}^{t/\tau}-p_{z}\log\left(q\right)q^{t/\tau}\right|dt .
\end{align*}
We condition on the case that $\left|\hat{q}-q\right|\le\epsilon_{\mathrm h}q$
and $\mathrm{TV}\left(\mathbf{p},\hat{\mathbf{p}}\right)\le\epsilon_{\mathrm t}$
which both happen with high probability due to
Lemma~\ref{lem:holding} and Lemma~\ref{lem:transition},
respectively.
We can then apply the bound
\begin{align}
\label{eq:9}
\big|\hat{a}\hat{b}-ab\big| =\big|\hat{a}(\hat{b}-b)+(\hat{a}-a)b\big|
 \le|\hat{b}-b|\cdot|\hat{a}|+|\hat{a}-a|\cdot|b|
\end{align}
twice to each inner term. That is, we first bound
\begin{align*}
&\big|{\hat{q}^{t/\tau}}{\log\left(\hat{q}\right)}-{q^{t/\tau}}{\log\left(q\right)}\big| \\
 &\le\big|\log\left(\hat{q}\right)-\log(q)\big|\cdot\hat{q}^{t/\tau}+\big|\hat{q}^{t/\tau}-q^{t/\tau}\big|\cdot\left|\log\left(q\right)\right|\\
 & \le2\epsilon\hat{q}^{t/\tau}+\big|\hat{q}^{t/\tau}-q^{t/\tau}\big|\cdot\left|\log\left(q\right)\right|.
\end{align*}
since $\left|\log\left(\hat{q}\right)-\log\left(q\right)\right|=\log\left(\max\left\{ \frac{\hat{q}}{q},\frac{q}{\hat{q}}\right\} \right)\le e^{2\epsilon_{\mathrm h}}$.
We use this to bound
\begin{align*}
&\big|{\hat{p}_{z}}{\log\left(\hat{q}\right)\hat{q}^{t/\tau}}-{p_{z}}{\log\left(q\right)q^{t/\tau}}\big|\\
& \le\big|\log\left(\hat{q}\right)\hat{q}^{t/\tau}-\log\left(q\right)q^{t/\tau}\big|\cdot\hat{p}_{z}+\left|\hat{p}_{z}-p_{z}\right|\cdot\big|\log\left(q\right)q^{t/\tau}\big|\\
 & \le2\epsilon_{\mathrm h}\hat{q}^{t/\tau}\hat{p}_{z}+\big|\hat{q}^{t/\tau}-q^{t/\tau}\big|\cdot\left|\log\left(q\right)\right|\cdot\hat{p}_{z}+\left|\hat{p}_{z}-p_{z}\right|\cdot\big|\log\left(q\right)q^{t/\tau}\big| .
\end{align*}
Plugging this back in, we obtain
\begin{align*}
\mathrm{TV}({K}_{y}^{\ell},\hat{{K}}_{y}^{\ell})
\le\frac{1}{2}\int_{0}^{\infty}\sum_{z\not=y}\Bigg(\underbrace{2\frac{\epsilon_{\mathrm h}}{\tau}\hat{q}^{t/\tau}\hat{p}_{z}}_{(\mathrm I)}
\quad +\underbrace{\left|\hat{q}^{t/\tau}-q^{t/\tau}\right|\cdot\frac{\left|\log\left(q\right)\right|}{\tau}\cdot\hat{p}_{z}}_{(\mathrm{II})}+\underbrace{\left|\hat{p}_{z}-p_{z}\right|\cdot\left|\frac{\log\left(q\right)}{\tau}q^{t/\tau}\right|}_{(\mathrm{III})}\Bigg)dt.
\end{align*}
We analyze all three error terms separately. First, we bound
\begin{align*}
(\mathrm{I})=2\frac{\epsilon_{\mathrm h}}{\tau}\sum_{z\not=y}\hat{p}_{yz}\int_{0}^{\infty}\hat{q}^{t/\tau}dt
 =2\epsilon_{\mathrm h}\frac{1}{\left|\log\left(\hat{q}\right)\right|}
 \le2\epsilon_{\mathrm h}\frac{1}{\left|\log\left(q\right)\right|-\log\left(1+\epsilon_{\mathrm h}\right)}
 =2\frac{\epsilon_{\mathrm h}}{\left|K_{yy}\right|\tau-\epsilon_{\mathrm h}}
\end{align*}
and, assuming that $\sqrt{\epsilon_{\mathrm h}}\le\kappa^{-1}$,
\begin{align*}
(\mathrm{II})&=\int_{0}^{\infty}\sum_{z\not=y}\left|\hat{q}^{t/\tau}-q^{t/\tau}\right|\cdot\frac{\left|\log\left(q\right)\right|}{\tau}\cdot\hat{p}_{z}dt \\
 & =\left|K_{yy}\right|\cdot\int_{0}^{\infty}\left|\hat{q}^{t/\tau}-q^{t/\tau}\right|dt\\
 & \le\left|K_{yy}\right|\left(\frac{1}{\left|K_{yy}\right|-\frac{\epsilon_{\mathrm h}}{\tau}}-\frac{1}{\left|K_{yy}\right|}\right)\\
 & =\frac{\epsilon_{\mathrm h}}{\left|K_{yy}\right|\tau-\epsilon_{\mathrm h}}.
\end{align*}
Finally,
\begin{align*}
(\mathrm{III})&=\int_{0}^{\infty}\sum_{z\not=y}\left|\hat{p}_{z}-p_{z}\right|\cdot\left|\frac{\log\left(q\right)}{\tau}q^{t/\tau}\right|dt \\
& =\sum_{z\not=y}\left|\hat{p}_{z}-p_{z}\right|\cdot\left|K_{yy}\right|\int_{0}^{\infty}e^{K_{yy}t}dt\\
 & =\sum_{z\not=y}\left|\hat{p}_{z}-p_{z}\right|\\
 & =2\mathrm{TV}\left({\bf p},\hat{{\bf p}}\right).
\end{align*}
Thus, we incur a total error of
\[
\mathrm{TV}({K}_{y}^{\ell},\hat{{K}}_{y}^{\ell})
 \le2\frac{\epsilon_{\mathrm h}}{\left|K_{yy}\right|\tau-\epsilon_{\mathrm h}}+\mathrm{TV}\left({\bf p},\hat{{\bf p}}\right)\\
 \le12\kappa\sqrt{\epsilon_{\mathrm h}}+\epsilon_{\mathrm t}
\]
since $\left|K_{yy}\right|\tau-\epsilon_{\mathrm h}\ge\frac{1}{6}\sqrt{\epsilon_{\mathrm h}}\kappa^{-1}$
if $6\sqrt{\epsilon_{\mathrm h}}\le\kappa^{-1}$. We can thus set $\sqrt{\epsilon_{\mathrm h}}=\frac{\epsilon}{24\kappa}$
and $\epsilon_{\mathrm t}=\frac{\epsilon}{2}$ to obtain $\mathrm{TV}({ K}_{y},\hat{{K}}_{y})\le\epsilon$.
With Lemma~\ref{lem:holding} and
Lemma~\ref{lem:transition}, this
gives us the bound on $m$
in the theorem statement.
\end{proof}

\section{Proofs from Section~\ref{sec:clustering}}

\spara{Long trails}
\label{sec:appendix-long-length}%
\label{subsec:long-trails}%
In order to apply the theorem
of \cite{kausik2023mdps}, we
need to ensure that each
pair of chains
exhibits a model difference
$\Delta$ in a state that is
visited sufficiently often.
To ensure the former, we
show how to transfer a model
difference from the mixture
of CTMCs to a mixture of
discrete-time Markov chains:

\begin{lemma}
\label{lem:kausik}
For any two rate matrices
$K, K' \in \{K^1, \dots, K^L\}$
and a state $y \in [n]$,
if $\| K_y - K'_y \|_2 \ge
\frac 1 \tau \Delta + 8 \tau (1 + K_{\max}^2)$ then
$\| e^{K \tau}_y - e^{K' \tau}_y \|_2 \ge \Delta$.
\end{lemma}

\begin{proof}
By the definition of the
matrix exponential as
a Taylor series,
\[
    e^{K \tau} - e^{K' \tau}
    = \tau (K - K') -
    \sum_{k=2}^\infty \frac 1 {k!} ((K \tau)^k  - (K' \tau)^k)
\]
and thus, by the triangle
inequality,
\begin{align}
    \label{eq:12}
    \|e^{K \tau}_y - e^{K' \tau}_y\|_2
    \ge \tau \| K_y - K'_y \|_2
    - \| A_y \|_2 - \| A'_y \|_2
\end{align}
where
$A \coloneqq \sum_{k=2}^\infty
= (K \tau)^k$
and $A'$ is defined analogously.
It thus remains to analyze
$\| A_y \|$.
To this end, let again
$K_{\max} \coloneqq
\max_{\ell, y} |K^\ell_{yy}|$.
Since rate matrices
are diagonally
dominant, we know by the
Gershgorin circle theorem
that all eigenvalues
of $K \tau$ lie in
$[-2 K_{\max} \tau, 0]$.
It is easy to see that the
magnitude of the eigenvalues
of $A$ are thus bounded by
\[
    |\lambda(A)|
    \le \sum_{k=2}^\infty 
    \frac 1 {k!} (2 K_{\max} \tau)^k
    = e^{2 K_{\max} \tau} - 1 - 2 K_{\max} \tau 
    \le 4 K_{\max}^2 \tau^2
\]
for $K_{\max} \tau \le \frac 1 2$
and therefore
$\| A_y \|, \| A'_y \| \le 4 K_{\max}^2 \tau^2$.
We plug this back into
\eqref{eq:12} and obtain
$
    \|e^{K \tau}_y - e^{K' \tau}_y\|_2
    \ge \tau \| K_y - K'_y \|_2
    - 8 K_{\max}^2 \tau^2
    \ge \Delta
$.
\end{proof}

When setting
$\tau^{-1} = \Theta( \kappa K_{\max} )$ in the
as in the context
of Section~\ref{subsec:discretization},
we obtain:
If
$
    \| K_y - K'_y \|_2
    = \kappa \Omega\left(
    \kappa K_{\max} \Delta
    + K_{\min}
    \right)
$
then
$\| e^{K \tau}_y - e^{K' \tau}_y \|_2 \ge \Delta$.

\spara{Very Long Trails}
\label{sec:appendix-very-long-length}%
%
%
We define the $\pi$-norm of
a vector $u \in \mathbb R^n$ through
\[
    \| u \|_\pi^2 =
    \sum_{i=1}^n \frac{u_i^2}{\pi(i)} .
\]

\begin{theorem}[Chernoff-Hoeffding
Bound for Random Walks
\cite{chung2012chernoffhoeffding}]
    Let $\mathbf x \in [n]^m$ be a random walk
    of length $m$ in a discrete-time
    Markov chain $M \in \R^{n \times n}$
    with associated starting
    probabilities $s \in [0,1]^n$.
    Let the stationary distribution
    of $M$ be $\pi$ and
    the mixing time $t_{\mathrm{mix}}$.
    Let $f \colon [n] \to [0, 1]$ be a
    function with
    $\mu \coloneqq \mathbb E_{y \sim \pi}[f(y)]$
    and $F \coloneqq \sum_{z \in \mathbf x} f(z)$.
    Then,
    \[
        \Pr[F \ge (1 + \delta) \mu m]
        \le \| s \|_\pi \cdot
        \left\{ \begin{array}{ll}
        e^{-\Omega(\delta^2 \mu m / t_{\mathrm{mix}})}
        & \mathrm{for}\textrm{ } 0 \le \delta \le 1 \\
        e^{-\Omega(\delta \mu m / t_{\mathrm{mix}})}
        & \mathrm{for}\textrm{ } \delta \ge 1 .
        \end{array} \right.
    \]
\end{theorem}

For some fixed state $y$,
let $f(z) = 0$ if
$z = y$ and $f(z) = 1$, otherwise,
such that
$c_y \coloneqq \sum_{z \in \mathbf x} f(z) = m - F$
is the number of times
a random walk $\mathbf x$
traverses state $y$.
We bound the probability
of obtaining less than
$\theta$ samples:
\begin{align}
    \label{eq:3}
    \Pr[c_y \le \theta]
    \le \Pr[F \ge (1 + \delta) \mu m]
    \le \| s \|_\pi e^{- \Omega(\delta \mu m / t_{\mathrm{mix}})}
\end{align}
for $\mu = 1 - \pi(y)$ and
$\delta = \Omega(\frac{\pi(y) - \theta / m}\mu)$.

\begin{theorem}
\label{thm:longtrails}
Given a Markov chain with transition probabilities $M$, we can
learn $M$ up to recovery error $\epsilon$ from a single trail of length
\[
    m = \Omega\left( 
        \frac{n \epsilon^{-2} + t_{\mathrm{mix}}(M)}{\pi_{\mathrm{min}}}
        \log \frac n {\pi_{\mathrm{min}}}
    \right)
\]
with high probability.
\end{theorem}

Let $M_y = (M_{yz}) \in \R^n$ be the vector of transition
probabilities out of state $y$.
We define an estimator $\hat M_y$ 
through
$\hat M_{yz} \coloneqq c_{yz} / c_y$,
where $c_{yz}$ is the number of
times the trail transitions from
state $y$ to $z$ and
$c_y \coloneqq \sum_{z} c_{yz}$.

\begin{proof}
We first bound the TV-distance from $M_y$ to $\hat M_y$ for
a fixed count $c_y$.
Choose an arbitrary set of states $Z \subseteq [n]$ and let
$M_y(Z) \coloneqq \sum_{z\in Z} M_{yz}$. By a Hoeffding bound,
\[
    \Pr\left[| M_y(Z) - \hat M_y(Z) | > \epsilon \right]
    \le 2 e^{- 2 c_y \epsilon^2} .
\]
By a union bound,
the probability that any set $Z$ experiences
error more than $\epsilon$ under
$\theta$ observations is at most
\begin{align}
    \label{eq:4}
    \Pr\left[\TV(M_y, \hat M_y) > \epsilon \right]
    \le 2^n \cdot 2 e^{- 2 c_y \epsilon^2}
    \le e^{n - 2 c_y \epsilon^2}
\end{align}
for sufficiently large $n$.
Let now
\[
    \theta = O\left( m \pi(y)
    + \frac{n - m \pi(y) \epsilon^2}{1 / t_{\mathrm{mix}} + \epsilon^2}
    \right)
\]
where $t_{\mathrm{mix}} \coloneqq t_{\mathrm{mix}}(M)$.
We consider two bad events. First, that
there are not enough samples, i.e. $c_y \le \theta$.
Second, that given at least $\theta$ samples,
the estimation error is larger than $\epsilon$.
Using \eqref{eq:3},
the probability of the first bad event is at most
\[
    \Pr[c_y \le \theta] \le
    \| s \|_\pi e^{-\Omega(\delta \mu m / t_{\mathrm{mix}})}
    = \| s \|_\pi \exp \left(
        -\Omega\left(\frac{ m \pi(y) - n \epsilon^{-2} }
          { \epsilon^{-2} + t_{\mathrm{mix}} }\right) 
    \right) .
\]
Furthermore, the probability of the
second bad event is, due to
\eqref{eq:4},
\[
    \Pr[ \TV(M_y, \hat M_y) > \epsilon \mid c_y \ge \theta]
    \le \exp \left(
        -\Omega\left(\frac{ m \pi(y) - n \epsilon^{-2} }
          { \epsilon^{-2} + t_{\mathrm{mix}} }\right) 
    \right) .
\]
Combining both, we see that
probability that of
error $\ge \epsilon$ when
estimating $M_x$ is at most
\begin{align*}
    \Pr[\TV(M_y, \hat M_y) > \epsilon]
    &\le \| s\|_\pi \exp \left(
        -\Omega\left(\frac{ m \pi(y) - n \epsilon^{-2} }
          { \epsilon^{-2} + t_{\mathrm{mix}} }\right)
    \right) \\
    &\le \frac 1 {\sqrt{\pi_{\mathrm{min}}}}
        \exp\left( - \Omega\left(
            \frac{m \pi_{\mathrm{min}} - n \epsilon^{-2}}{\epsilon^{-2} + t_{\mathrm{mix}}}
        \right) \right) .
\end{align*}
By a union bound, the probability that
for a fixed chain, 
any state has estimation
error more than $\epsilon$ is
at most
\[
    \frac n {\sqrt{\pi_{\mathrm{min}}}} \exp \left(
        -\Omega\left(\frac{ m \pi_{\mathrm{min}} - n \epsilon^{-2} }
          { \epsilon^{-2} + t_{\mathrm{mix}} }\right) 
    \right) 
\]
Therefore, with high probability, we are
able to estimate the chain from a single trail
within error $\epsilon$ when setting
$m$ as the the theorem statement.
\end{proof}
Finally, we remark that the probability that there is
a chain for which we do not obtain
a single trail is, by a union bound,
at most
\[
    \sum_{\ell=1}^L (1 - \| s^\ell \|_1)^m
    = L \cdot (1 - \Omega(1 / L))^m
    \le L \cdot e^{- \Omega(m / L)} .
\]
It thus suffices to set
$m = \Omega(L \log L)$
to obtain samples from ever
chain with high probability.

\section{Proofs from Section~\ref{subsec:recovery}}
\label{sec:appendix-recovery}

{\sc Theorem}~\ref{thm:amgm}.
\emph{
For each $\mathbf x \in \mathbf X$,
\begin{align*}
    \textstyle
    \prod_{\ell=1}^L
    \Pr[\mathbf x \mid \mathbf K \cap \ell]^{a (\mathbf x, \ell)}
    \le
    \sum_{\ell=1}^L
    a (\mathbf x, \ell) \cdot \Pr[\mathbf x \mid \mathbf K \cap \ell]
    \textstyle
    \le L \cdot (\max_\ell a(\mathbf x, \ell)) \cdot
    \prod_{\ell=1}^L
    \Pr[\mathbf x \mid \mathbf K \cap \ell]^{a (\mathbf x, \ell)} .
\end{align*}
}

\begin{proof}[Proof (Theorem~\ref{thm:amgm})]
    The first inequality can be directly derived from
    the arithmetic-geometric mean inequality
    using the factors $\hat a(\mathbf x, \ell)$.
    For the second, it is worth noting that without loss of generality 
    we can assume
    $\sum_{\ell=1}^L \Pr[\mathbf x \mid \mathbf K \cap \ell] = 1$. This is because both the arithmetic and geometric mean
    scale linearly. We thus want to show that
    \begin{align}
        \label{eq:8}
        \textstyle
        \sum_{\ell=1}^L p_\ell^2 \le
        C \cdot \prod_{\ell=1}^L p_\ell^{p_\ell}
    \end{align}
    for $p_\ell = \hat a(\mathbf x, \ell)
     = \Pr[\mathbf x \mid \mathbf K \cap \ell]$.
    Note that \eqref{eq:8} is equivalent to
    \[
        2^{\log_2 \sum_{\ell=1}^L p_\ell^2}
        =
        2^{- H_2(\mathbf p)}
        \le C \cdot 2^{- H_1(\mathbf p)}
        = C \cdot 2^{\sum_{\ell=1}^L p_\ell \log_2 p_\ell}
    \]
    where $H_\alpha(\mathbf p)$ is the \Renyi{} entropy\footnote{The
    \Renyi{} entropy of order $\alpha$ is defined as
    $H_\alpha(\mathbf p) = \frac 1 {1 - \alpha} \log_2\left( \sum_{\ell=1}^L p_i^\alpha \right)$
    while the Shannon entropy $H_1(\mathbf p)$ and min-entropy $H_\infty(\mathbf p) = -\log_2 \max_\ell p_\ell$ are defined trough the limit.}
    of $\mathbf p$ and a suitable constant $C > 0$.
    By well-known facts about the \Renyi{}-entropy,
    \begin{align*}
        H_1(\mathbf p) - H_2(\mathbf p) \le \log_2 L - H_2(\mathbf p)
        \le \log_2 L - H_\infty(\mathbf p)
        = \log_2 L + \log_2 \max_\ell p_\ell
    \end{align*}
    so we obtain $C = L \cdot \max_\ell p_\ell$
    as required for the theorem statement.
\end{proof}

\end{document}